%% file: arxiv.tex
\documentclass{article} 
\usepackage{iclr2027_conference_preprint,times}

\input{math_commands.tex}

\usepackage{changepage}
\usepackage{algorithm}
\usepackage{textcomp,gensymb}
\usepackage{algpseudocode}
\usepackage[dvipsnames,table]{xcolor}
\usepackage{mathtools}
\usepackage{makecell}
\usepackage{multirow}
\usepackage{graphicx}

\definecolor{linkblue}{RGB}{0,0,180}
\definecolor{citegreen}{RGB}{60,160,90}
\definecolor{reflink}{RGB}{160,50,60}
\usepackage[colorlinks,linkcolor=reflink,urlcolor=bookorange,citecolor=citegreen]{hyperref}

\usepackage{wrapfig}
\usepackage{booktabs}
\usepackage{amssymb}
\usepackage{pifont}
\usepackage{url}
\usepackage{cleveref}
\usepackage{float}

\Crefname{equation}{Eq.}{Eqs.}
\creflabelformat{equation}{(#2#1#3)}
\Crefname{figure}{Fig.}{Figs.}
\crefname{appendix}{appendix}{appendices}
\Crefname{appendix}{Appendix}{Appendices}
\definecolor{bookorange}{HTML}{D96C18}
\usepackage{amsthm}

\newtheorem{prop}{Proposition}
\crefname{prop}{proposition}{propositions}
\Crefname{prop}{Proposition}{Propositions}

\usepackage[most]{tcolorbox}

\newtcolorbox{keyformula}{
    enhanced,
    colback=black!4,
    colframe=black!100,
    boxrule=0.8pt,
    arc=1.5mm,
    left=1.0mm,
    right=1.0mm,
    top=0.8mm,
    bottom=0.8mm,
    before skip=5pt,
    after skip=5pt,
    fontupper=\small,
    before upper={
        \setlength{\abovedisplayskip}{0pt}
        \setlength{\belowdisplayskip}{0pt}
        \setlength{\abovedisplayshortskip}{0pt}
        \setlength{\belowdisplayshortskip}{0pt}
    }
}

\title{Learning What to Recall: Adaptive Multi-\\Cue Episodic Memory for World Models}

\author{%
\parbox[t]{\textwidth}{%
\vspace*{-2mm}
\centering\normalfont
\begin{tabular*}{\textwidth}{
  @{\extracolsep{\fill}}ccc@{}
}
\textbf{Beomsu Kim}$^{1}$ &
\textbf{Chieh-Hsin Lai}$^{2}$ &
\textbf{Bac Nguyen}$^{2}$ \\
\textbf{Amir Bar}$^{3}$ &
\textbf{Jong Chul Ye}$^{1,\dagger}$ &
\textbf{Yuki Mitsufuji}$^{2,\dagger}$ \\[0.4em]
$^{1}$KAIST &
$^{2}$Sony Group Corporation &
$^{3}$Imperial College London
\end{tabular*}\\[0.5em]}}

\newcommand{\xmark}{\textcolor{red!45!black}{\ding{55}}}
\newcommand{\cmark}{\textcolor{green!45!black}{\ding{51}}}

\renewcommand{\eqref}[1]{Eq.~(\ref{#1})}

\newcommand{\bma}{{\bm{a}}}

\newcommand{\bmf}{{\bm{f}}}
\newcommand{\bmk}{{\bm{k}}}

\newcommand{\bmo}{{\bm{o}}}
\newcommand{\bmq}{{\bm{q}}}

\newcommand{\bmw}{{\bm{w}}}
\newcommand{\bmx}{{\bm{x}}}
\newcommand{\bmy}{{\bm{y}}}
\newcommand{\bmz}{{\bm{z}}}

\newcommand{\cB}{\mathcal{B}}
\newcommand{\cC}{\mathcal{C}}

\newcommand{\cL}{\mathcal{L}}
\newcommand{\cM}{\mathcal{M}}
\newcommand{\cQ}{\mathcal{Q}}
\newcommand{\cZ}{\mathcal{Z}}
\newcommand{\cR}{\mathcal{R}}

\newcommand{\bbE}{\mathbb{E}}

\newcommand{\emdr}{EMDR\textsuperscript{2}}

\DeclareMathOperator{\vision}{vision}
\DeclareMathOperator{\audio}{audio}
\DeclareMathOperator{\ret}{Ret}

\DeclareMathOperator{\wm}{WM}

\DeclareMathOperator{\nce}{NCE}
\DeclareMathOperator{\mlp}{MLP}
\DeclareMathOperator{\sg}{sg}

\DeclareMathOperator{\enc}{enc}
\DeclareMathOperator{\chunk}{chunk}
\DeclareMathOperator{\lazy}{lazy}

\DeclareMathOperator{\topk}{TopK}
\DeclareMathOperator{\ctopk}{ChunkTopK}
\DeclareMathOperator{\cossim}{CosSim}

\iclrpreprintcopy
\begin{document}

\maketitle

\begingroup
\makeatletter
\renewcommand{\@makefntext}[1]{\noindent#1}
\makeatother
\renewcommand{\thefootnote}{}
\footnotetext[0]{$^\dagger$Corresponding Authors.\\ Work done while Beomsu Kim was an intern at Sony. \\ Contacts: \texttt{beomsu.kim@kaist.ac.kr}, \texttt{chieh-hsin.lai@sony.com}}
\endgroup

\begin{figure}[H]
\vspace{-6mm}
\centering
\includegraphics[width=1.0\linewidth]{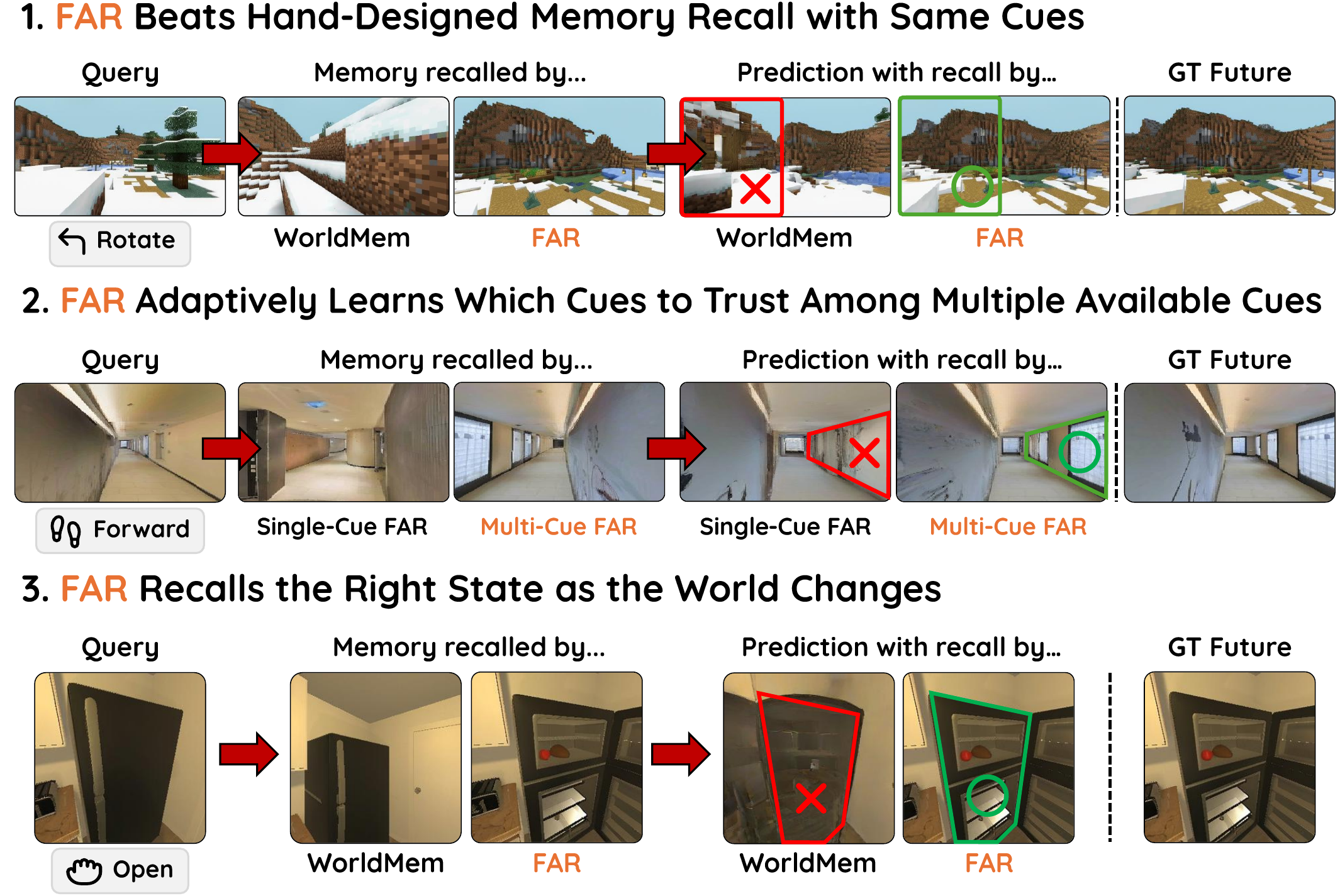}
\vspace{-5mm}
\caption{\textbf{The three advantages of Future-Aware Recall (FAR).} \textbf{Top:} FAR retrieves context based on its usefulness for predicting the future, beating hand-designed memory recall rules. In the figure, FAR is compared against WorldMem~\citep{xiao2025worldmem}. \textbf{Middle:} FAR adaptively learns query-dependent weights over multiple available retrieval cues. Here, FAR learns to rely more on audio cue to distinguish memories with occluded and clear line of sight.
 \textbf{Bottom:} FAR recalls memory to support state-correct future predictions. Here, FAR distinguishes between closed and open fridges, in contrast to WorldMem.}
\vspace{-3mm}
\label{fig:teaser}
\end{figure}

\begin{abstract}
\vspace{-2mm}

World models predict future observations from current experience and actions, yet prediction can depend on observations seen far in the past.
Episodic memory preserves past observations for later recall; however, as
memory accumulates, it raises a fundamental question:
\emph{which memories are useful for the current prediction, and which available
retrieval cues should be trusted to find them?}
This is challenging because fixed criteria based on recency, pose overlap, or
visual similarity can be unreliable across environments and queries.
We propose {\color{bookorange}\textbf{{Future-Aware Recall}} (\textbf{FAR})}, a
framework that learns episodic recall from future-aware predictive supervision
and adaptive multi-cue scoring.
During training, FAR measures predictive utility by the conditional
log-likelihood of the realized future given recalled context, approximated by
negative diffusion prediction loss, and uses it to
train a retriever that remains future-blind at inference.
The retriever learns cue-specific relevance and automatically determines which
available retrieval cues, such as time, pose, vision, and audio, to trust for
each query when selecting memories.
Across three complementary settings, FAR outperforms hand-designed recall even
with the same retrieval cues, automatically adapts which available cues to
trust, and recalls the right history as the world changes. Together, these results establish FAR as a flexible, principled approach to
episodic memory access in world models.
Project page at: {\color{bookorange}\textbf{\url{https://1202kbs.github.io/FAR-Project-Page/}}}
\end{abstract}

\vspace{-4mm}
\section{Introduction}

In world modeling, the current observation provides only a partial view of the
underlying environment state.
Scenes and objects persist, and may continue to evolve, after leaving the field
of view.
To predict what will be observed upon revisit, world models must preserve
information across extended interactions
~\citep{ha2018wm,hafner2019planet}.
Memory is therefore central to long-horizon, persistent world modeling. Recent world models explore different mechanisms for maintaining such information.
\emph{Persistent-state memory} integrates observations into an evolving
representation of the current world, such as a recurrent latent state
~\citep{hafner2020dreamerv1,hafner2021dreamerv2,hafner2025dreamerv3}
or persistent 3D representation
~\citep{wu2025spmem,garcin2026persist}.
\emph{Episodic memory} instead preserves individual past observations as
separately accessible memories that can later be recalled
~\citep{xiao2025worldmem,hu2026longliverag}.
These memory functions are complementary:
persistent-state memory provides compact access to an evolving world state,
while episodic memory preserves directly recoverable evidence from past
experience.

Episodic memory, however, introduces a fundamental recall problem.
As interaction continues, the number of stored memories grows, while only a
small fraction may be useful for a particular prediction.
Existing world models commonly define relevance using fixed criteria such as
temporal recency~\citep{bar2025nwm}, pose-based field-of-view overlap
~\citep{xiao2025worldmem}, or visual embedding similarity
~\citep{hu2026longliverag}.
However, similarity under a particular cue need not reflect predictive
usefulness, and the reliability of that cue can vary across situations.
For example, in an elbow-shaped corridor
(see \Cref{fig:ss_retrieval}), pose-based retrieval may favor a nearby
memory across a wall, while visual appearance may be ambiguous across similar
corridor segments.
Spatial audio may instead better identify relevant past experience in this
case, while pose or vision may be more informative elsewhere.
This raises our central question:
\begin{adjustwidth}{1.5em}{1.5em}
\emph{
How can a world model learn which past observations are useful for future
prediction, and which available retrieval cues can identify them before the
future is known?
}
\end{adjustwidth}
We address this with {\color{bookorange}\textbf{{Future-Aware Recall}} (\textbf{FAR})}.
During training, the future is observed, allowing FAR to evaluate each candidate
memory by how well it helps the world model predict what actually happens.
We call this \emph{predictive utility}, and use it to supervise a retriever that
must operate without access to the future at inference.
To predict this relevance at recall time, FAR learns a relevance function for
each available cue, such as time, pose, vision, or audio, together with
query-dependent weights that determine which cues to trust.
These cues are used to select memories; they need not be provided to the world
model as additional generation conditions.
FAR's latent-variable formulation connects naturally to retrieval-augmented
language models~\citep{lewis2020rag,sachan2021emdr}, allowing us to adapt
discrete retriever optimization to episodic recall for world models.

We instantiate FAR with video diffusion world models and an external episodic
memory.
Historical observations remain individually addressable, while the retriever
selects a compact Top-$K$ context for each prediction.
Predictive utility is defined through the conditional likelihood of the realized
future; for our video diffusion instantiation, we use negative diffusion
prediction loss as a tractable surrogate.
The recalled context then conditions the world model for future prediction.

Across three complementary environments, FAR consistently improves episodic
recall and downstream prediction over fixed retrieval strategies.
In LoopNav~\citep{lian2025loopnav}, FAR outperforms hand-designed relevance
rules even when given the same retrieval cues, while adaptively using time,
pose, and vision further improves long-horizon prediction.
In SoundSpaces~\citep{chen2020soundspaces}, FAR learns to rely on audio when
spatial cues become ambiguous, particularly for longer trajectories and sparser
memories.
In changing-state AI2-THOR environments~\citep{kolve2017ai2thor}, FAR recalls
history consistent with the current world state and reduces errors caused by
stale memories.
\Cref{fig:teaser} summarizes these complementary capabilities.
Together, these results demonstrate FAR as a flexible, prediction-driven
mechanism for episodic memory access in persistent world models.

\vspace{-2mm}
\section{Background and Related Work}
\label{sec:background}
\vspace{-2mm}

\textbf{World models and episodic memory.}
World models predict future observations from interaction history and actions
using recurrent dynamics, diffusion, or masked generative modeling
~\citep{hafner2020dreamerv1,hafner2021dreamerv2,hafner2025dreamerv3,
alonso2024diamond,bruce2024genie,kim2026compact}.
Over long interactions, relevant information may lie far in the past, making
repeated processing of the full history costly.
External episodic memory instead preserves past observations for selective
recall using cues such as time, pose, vision, or audio.
We focus on learning which memories to recall and which available cues to trust, rather than relying on fixed relevance rules for individual cues.

\textbf{Internal and external recall.}
Generator-internal methods maintain long-range information through attention, routing, context compression, or recurrent memory
~\citep{cai2026moc,yu2026memlearner,peng2026car}, with recent world models also using linear attention for efficient long-horizon memory ~\citep{wang2026context,zhu2026sanawm}.
Generator-external methods instead search an explicit memory and pass only a compact subset to the world model ~\citep{xiao2025worldmem,yu2025cam,chen2025vrag,li2026i3dm}.
These approaches are complementary: internal memory compactly summarizes
history, while external episodic memory preserves individually addressable
observations for selective recall.
We focus on external recall, which can be combined with internal long-context memory mechanisms.

\textbf{Learning predictive external recall.}
External world-model recall commonly uses fixed temporal, geometric, or
embedding-based relevance, while discrete Top-$K$ selection prevents direct
gradient flow from the world-model objective.
Related retrieval-augmented generation methods learn discrete retrieval from
downstream likelihood; for example, \emdr~\citep{sachan2021emdr} uses reader
likelihood to supervise a document retriever.
FAR applies this latent-variable principle to episodic recall from an evolving
interaction history, defining relevance through future predictive utility.
Its retriever learns both cue-specific relevance and query-dependent cue
reliability from available cues whose usefulness can vary across queries.
\Cref{tab:related_memory_retrieval} in \Cref{app:related_work} provides a
detailed comparison.

\vspace{-2mm}
\section{Our Method: Future-Aware Episodic Memory}
\label{sec:method}
\vspace{-2mm}


We propose {\color{bookorange}
\textbf{{Future-Aware Recall}} (\textbf{FAR})}, a framework for learning predictive memory relevance in external episodic recall.
At inference, FAR must select useful memories using only the episodic memory and current prediction query, since the future to be predicted is unavailable.
During training, however, that future is observed.
FAR exploits this additional information to identify which memories would have best supported prediction and uses these signals to train a retriever that remains future-blind at inference.
FAR thereby learns both \emph{which memories to recall} and \emph{which retrieval cues to trust}. A high-level overview is provided in \Cref{fig:method}.
We first formalize the \emph{recall problem} underlying FAR, and then introduce its key mechanisms.

\textbf{The memory recall problem.}
Let $\bmo_t$ denote the observation at step $t$ and $\bma_t$ the subsequent
action, with prediction query
$\cQ_t\coloneqq(\bmo_t,\bma_t)$.
The \emph{episodic memory} contains the past observations
$\cM_{t-1}\coloneqq(\bmo_1,\ldots,\bmo_{t-1})$.
FAR recalls a compact context
$\cC_t\subseteq\cM_{t-1}$ with $|\cC_t|=K$, which the world model uses to predict
the next observation: $p_\theta(\bmo_{t+1}\mid\cC_t,\cQ_t)$.

To locate useful memories, each historical observation $\bmo_i$ is associated
with retrieval cues $\bmz_i$, such as time, pose, a visual representation, or an
audio representation.
We collect the historical cues as
$\cZ_{t-1}\coloneqq(\bmz_1,\ldots,\bmz_{t-1})$
and define the current retrieval query as
$\cR_t\coloneqq(\bmz_t,\bma_t)$.
These cues are used by the external retriever to determine where to recall
from in episodic memory.
Specifically, given $\cZ_{t-1}$ and $\cR_t$, the retriever assigns each candidate
context $\cC_t$ a \emph{relevance score}
$s_\phi(\cC_t\mid\cZ_{t-1},\cR_t)$, inducing the \emph{recall distribution}
\begin{align}
r_\phi(\cC_t\mid\cZ_{t-1},\cR_t)
\propto
\exp s_\phi(\cC_t\mid\cZ_{t-1},\cR_t).
\label{eq:recall_distribution}
\end{align}
Exact evaluation of \Cref{eq:recall_distribution} is generally intractable because it requires considering a combinatorial number of possible recalled context sets as the episodic memory $\cM_{t-1}$ grows. For discrete Top-$K$ recall, we therefore use the candidate-level latent-variable approximation adapted from \emdr~\citep{sachan2021emdr}, described in \Cref{app:retriever_posterior_approx}.

Together, the retriever and world model define the latent-context predictive
model
\begin{align}
\textstyle
p_{\theta,\phi}(\bmo_{t+1}\mid
\cM_{t-1},\cZ_{t-1},\cR_t,\cQ_t)
=
\sum_{\cC_t}
p_\theta(\bmo_{t+1}\mid\cC_t,\cQ_t)\,
r_\phi(\cC_t\mid\cZ_{t-1},\cR_t).
\label{eq:latent_context_model}
\end{align}
Here, $r_\phi$ determines which past observation is recalled, while $p_\theta$
determines how recalled context supports prediction.
Importantly, retrieval cues guide the selection of $\cC_t$; they need not be
provided to the world model as generation conditions.
Learning effective episodic recall thus reduces to learning the relevance
function $s_\phi$, and hence the recall distribution $r_\phi$.
We now derive predictive supervision for relevance and show how FAR
adaptively uses multiple retrieval cues to estimate it.

\begin{figure}[t]
\centering
\includegraphics[width=1.0\linewidth]{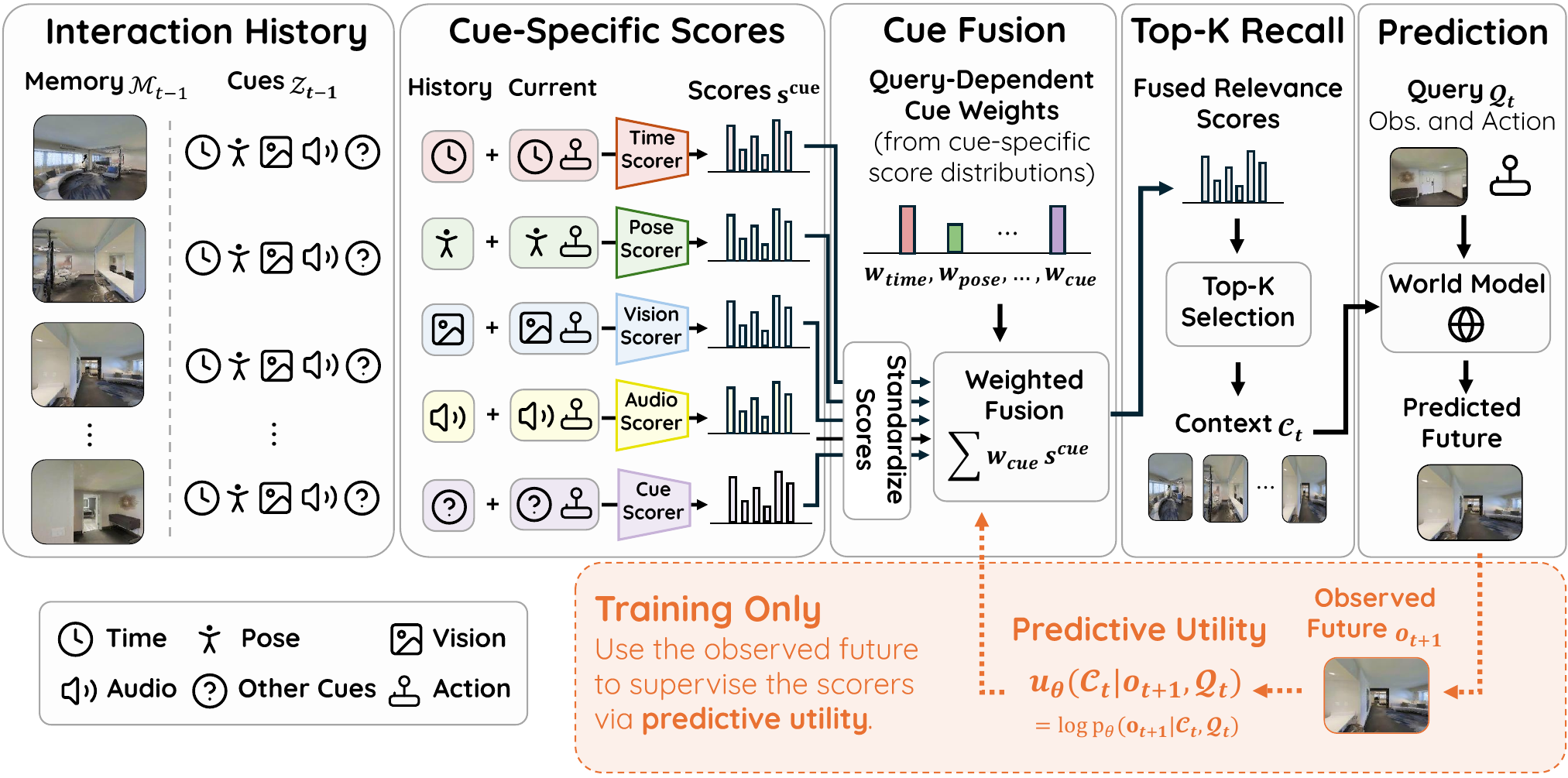}
\vspace{-7mm}
\caption{\textbf{Overview of Future-Aware Recall (FAR).} FAR scores memories using general retrieval cues $\cZ_{t-1}$, adaptively fuses cue-specific relevance, and recalls a compact Top-$K$ context $\cC_t$ from episodic memory $\cM_{t-1}$ for world-model prediction. During training, the observed future $\bmo_{t+1}$ provides predictive utility to supervise the retriever, while recall remains future-blind at inference.}
\vspace{-3mm}
\label{fig:method}
\end{figure}

\subsection{Learning Which Memories to Recall}
\label{sec:predictive_relevance}

The recall model above is governed by the future-blind relevance score
$s_\phi(\cC_t\mid\cZ_{t-1},\cR_t)$.
The key question is how this score should be learned so that highly ranked
contexts are actually useful for prediction.
Our principle is predictive:
\emph{a recalled context is useful when it provides information about the future
beyond what is already available in the current prediction query.}

\textbf{Predictive utility.}
Let $p_{\rm data}$ denote the environment data-generating distribution over
interaction histories, retrieval cues, and future observations.
For a recalled context
$\cC_t\sim r_\phi(\cdot\mid\cZ_{t-1},\cR_t)$,
its predictive information about the next observation is naturally measured by $I_\phi(\bmo_{t+1};\cC_t\mid\cQ_t)$, where the subscript emphasizes that $\cC_t$ is selected by the retriever
$r_\phi$.
$I_\phi$ measures how much the recalled context reduces uncertainty about
the future beyond what is already known from $\cQ_t$.
The following proposition connects this information-theoretic notion of relevance
to the world model $p_\theta$.

\begin{prop}[Predictive relevance bound]\label{prop:relevance-bound}
For any predictive distribution
$p_\theta(\bmo_{t+1}\mid\cC_t,\cQ_t)$,
\begin{align}
I_\phi\!\left(\bmo_{t+1};\cC_t\mid\cQ_t\right)
\geq
\bbE\!\left[
\log p_\theta(\bmo_{t+1}\mid\cC_t,\cQ_t)
-
\log p_{\rm data}(\bmo_{t+1}\mid\cQ_t)
\right],
\label{eq:predictive_information_bound}
\end{align}
where the expectation is over training interactions from $p_{\rm data}$ and
$\cC_t\sim r_\phi(\cdot\mid\cZ_{t-1},\cR_t)$.
\end{prop}
See the derivation in
\Cref{app:predictive_utility_derivation}.
For a fixed training pair $(\bmo_{t+1},\cQ_t)$, the second term in
\Cref{eq:predictive_information_bound} depends only on $p_{\rm data}$, and is therefore identical across recalled contexts and independent
of the trainable parameters.
The remaining context-dependent term hence motivates the
\emph{predictive utility}:
\begin{keyformula}
\begin{align}
u_\theta(\cC_t\mid\bmo_{t+1},\cQ_t)
\coloneqq
\log p_\theta(\bmo_{t+1}\mid\cC_t,\cQ_t).
\label{eq:predictive_utility}
\end{align}
\end{keyformula}
A context has high predictive utility when conditioning on it makes the realized
future more likely under the world model.
Thus, memory relevance depends on the prediction being made: the same past
experience may be highly useful for one query and largely irrelevant for another.

\textbf{Future-aware predictive supervision.}
Predictive utility depends on the realized future $\bmo_{t+1}$, which is not
available when memories must be recalled at inference time.
During training, however, $\bmo_{t+1}$ is observed and reveals which recalled
contexts would have best supported the prediction.
FAR uses this additional information to construct a future-aware posterior over
contexts.
Combining the recall distribution in
\Cref{eq:recall_distribution}
with the predictive utility in
\Cref{eq:predictive_utility},
Bayes' rule gives
\begin{keyformula}
\begin{align}
q_{\theta,\phi}
\!\left(
\cC_t
\mid
\bmo_{t+1},
\cM_{t-1},
\cZ_{t-1},
\cR_t,
\cQ_t
\right)
\propto
\exp\!\left[
s_\phi(\cC_t\mid\cZ_{t-1},\cR_t)
+
u_\theta(\cC_t\mid\bmo_{t+1},\cQ_t)
\right].
\label{eq:future_posterior}
\end{align}
\end{keyformula}

This posterior combines two signals:
$s_\phi$ captures how relevant a context appears from information available
before the future is observed, while $u_\theta$ provides predictive credit from
the realized future.
Accordingly, $q_{\theta,\phi}$ serves as a future-aware teacher for the
future-blind recall distribution.
As shown in \Cref{app:retriever_objective}, this posterior provides the
maximum-likelihood training target for recall: it favors contexts that better
explain the observed future while retaining the relevance already inferred by
the future-blind retriever.
We therefore train $r_\phi$ to match this fixed target:
\begin{align}
\cL_{\ret}(\phi)
\coloneqq
\bbE\!\left[
\KL\!\left(
\sg\!\left[
q_{\theta,\phi}
\right]
\,\middle\|\,
r_\phi
\right)
\right],
\label{eq:posterior_distillation}
\end{align}
where $\sg[\cdot]$ is the stop-gradient, treating
$q_{\theta,\phi}$ as fixed during the retriever update.
This transfers predictive credit from the observed future into a retriever that
operates without the future at inference.

\subsection{Adaptively Learning Which Cues to Trust}\label{sec:multi_cue}

The previous subsection specifies what the context relevance score $s_\phi$ should learn through future-aware predictive credit.
At inference, however, this credit is unavailable, so $s_\phi$ must estimate relevance from the retrieval cues contained in $\cZ_{t-1}$ and $\cR_t$.
FAR is agnostic to the choice of cues: any signal associated with a historical observation that helps locate useful memories can be used for recall.
In our experiments, we consider time, pose, visual representations, and audio representations, with the available subset depending on the environment.
Because the reliability of these cues can vary across queries, FAR learns both cue-specific relevance and how strongly each available cue should contribute to the final context score.

\textbf{Cue-specific relevance.}
Let $\bmz_i^m$ denote the component of $\bmz_i$ corresponding to retrieval cue $m$.
We define the corresponding cue history and retrieval query as $\cZ_{t-1}^m\coloneqq(\bmz_1^m,\ldots,\bmz_{t-1}^m)$ and $\cR_t^m\coloneqq(\bmz_t^m,\bma_t)$.
For each historical observation $\bmo_i$, cue $m$ produces an observation-level relevance score
\begin{align}
s_{\phi,i}^{m}
\coloneqq
s_\phi^{m}\!\left(\bmz_i^{m};\cZ_{t-1}^{m},\cR_t^{m}\right).
\label{eq:cue_specific_score}
\end{align}
This score estimates how relevant the corresponding memory appears when viewed through cue $m$ alone.
Because different cues may produce scores on different numerical scales, we standardize each cue's scores across the episodic memory and denote the resulting scores by $\widetilde{s}_{\phi,i}^{m}$.

\textbf{Adaptive cue fusion.}
FAR combines the cue-specific scores using query-dependent weights,
\begin{align}
\textstyle 
s_{\phi,i} =
\sum_m
\lambda_{\phi,t}^{m}\,
\widetilde{s}_{\phi,i}^{m},
\qquad
\sum_m \lambda_{\phi,t}^{m}=1,
\label{eq:multi_cue_score}
\end{align}
where $\lambda_{\phi,t}^{m}$ is produced by a learned gate from the cue-specific retrieval signals available for the current query.
The fused observation-level scores define the context relevance score in \Cref{eq:recall_distribution} as
\begin{align}
\textstyle
s_\phi(\cC_t\mid\cZ_{t-1},\cR_t)
\coloneqq
\sum_{\bmo_i\in\cC_t}
s_{\phi,i}.
\label{eq:context_score}
\end{align}
Together, \Cref{eq:multi_cue_score,eq:context_score} parameterize the recall distribution $r_\phi(\cC_t\mid\cZ_{t-1},\cR_t)$ used by the future-aware posterior in \Cref{eq:future_posterior}.
For fixed-size recall with $|\cC_t|=K$, the maximum-score context is obtained by selecting the $K$ observations with the largest fused relevance scores,
\begin{align}
\textstyle \cC_{t,\topk} = \arg\max_{\cC_t\subseteq\cM_{t-1}}
s_\phi(\cC_t\mid\cZ_{t-1},\cR_t) \quad \text{so that} \quad |\cC_t|=K.
\label{eq:topk_recall}
\end{align}

\subsection{Video Diffusion World Model Instantiation}
\label{sec:video_instantiation}

We now instantiate FAR with a video diffusion world model, and we describe the central details here. A precise description of implementation is provided in \Cref{app:video_far_details}.

\textbf{Cacheable retrieval encoders.}
For high-dimensional cues such as vision and audio, cue-specific encoders map historical cues to compact keys and the current cue and action to a query embedding.
We contrastively pretrain and freeze the memory-side encoders so that historical keys can be computed once and cached, while lightweight query-side adapters and cue-fusion components remain trainable through FAR.

\textbf{Diffusion-based predictive utility.}
Because the exact likelihood in \Cref{eq:predictive_utility} is expensive for diffusion models, we approximate predictive utility by the negative diffusion prediction loss~\citep{ho2020denoising,kingma2021variational,song2021maximum,lai2025principles} of the realized future conditioned on each candidate memory, averaged over four diffusion timesteps.
Memories that yield lower prediction loss receive greater future-aware credit.

\textbf{Temporally structured recall.}
To avoid spending multiple Top-$K$ slots on redundant nearby frames, we partition memory into temporal chunks, retain the highest-scoring observation from each chunk, and recall the $K$ highest-scoring representatives.

\begin{figure}[t]
\vspace{-1mm}
\centering
\includegraphics[width=1.0\linewidth]{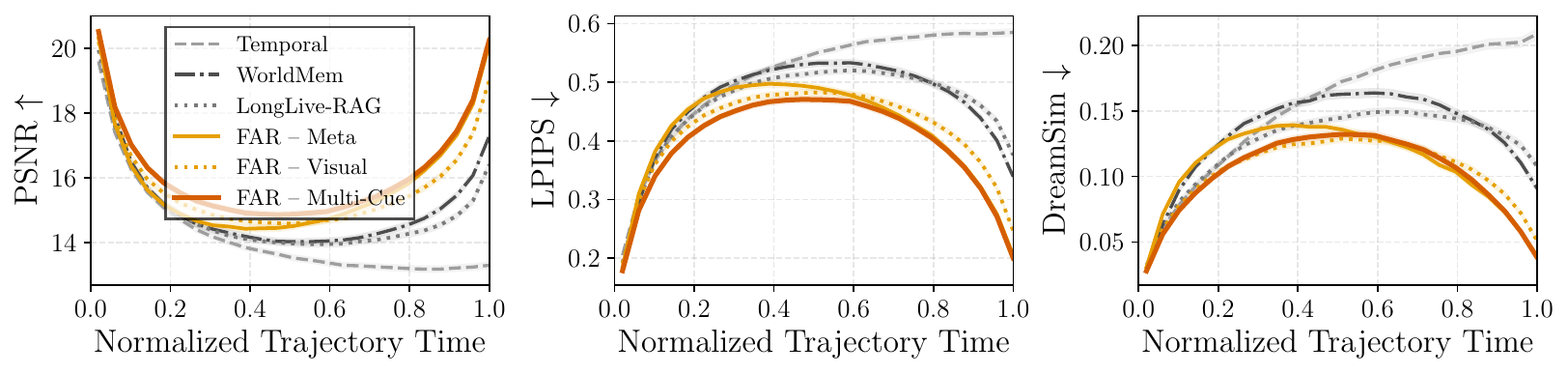}
\vspace{-6mm}
\caption{\textbf{LoopNav rollout quality.} We report frame-wise PSNR$\uparrow$, LPIPS$\downarrow$, and DreamSim$\downarrow$ between generated and ground-truth return trajectories as a function of normalized trajectory time. Curves show the mean across trajectories, with shaded regions indicating 95\% bootstrap confidence intervals. Errors are highest near the middle of the return phase, where relevant observations are typically farthest in time and memory is most sparse.}
\vspace{-4mm}
\label{fig:loopnav_rollout_metrics}
\end{figure}

\vspace{-3mm}
\section{Experiments}
\vspace{-3mm}

\begin{wrapfigure}{r}{0.25\linewidth}
\centering
\vspace{-4mm}
\includegraphics[width=1.0\linewidth]{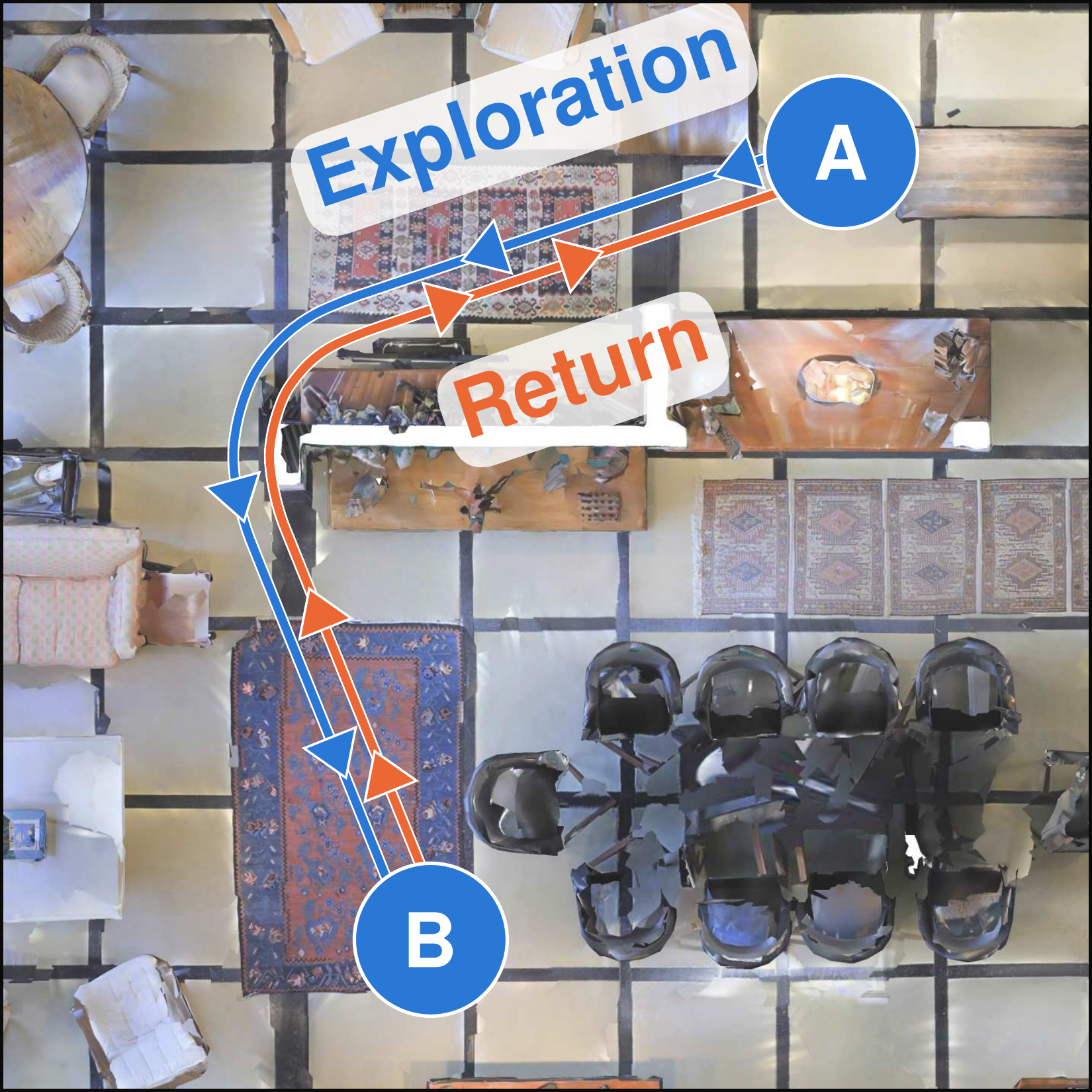}
\vspace{-5mm}
\caption{\textbf{Illustration of a data trajectory.}}
\label{fig:map}
\vspace{-2mm}
\end{wrapfigure}

We present experiments across three complementary world-modeling settings:
\emph{LoopNav}~\citep{lian2025loopnav}, \emph{SoundSpaces}~\citep{chen2020soundspaces}, and \emph{AI2-THOR}~\citep{kolve2017ai2thor}.
LoopNav evaluates consistency in a static Minecraft environment using metadata and visual retrieval cues, while SoundSpaces evaluates realistic indoor navigation where metadata and spatial audio provide complementary signals under partial observability.
AI2-THOR evaluates interactive household environments with manipulable objects, where interactions change the world state and memories may become stale.
Each trajectory consists of an \textit{exploration phase} and a \textit{return phase}.
The world model generates the return-phase video using exploration-phase observations as its memory pool for context retrieval.
\Cref{fig:map} shows exploration and return paths on SoundSpaces.

\vspace{-1mm}
\subsection{LoopNav -- FAR Outperforms Hand-Designed Recall}
\vspace{-2mm}

We compare against three retrieval baselines: the temporal baseline~\citep{bar2025nwm}, which uses the most recent observations; WorldMem~\citep{xiao2025worldmem}, which retrieves frames with high field-of-view (FOV) overlap with the goal; and LongLive-RAG~\citep{hu2026longliverag}, which retrieves by reconstructive embedding similarity.
We evaluate three variants of our method using metadata cues (time and pose), visual cues, or their adaptive fusion.

As shown in \Cref{fig:loopnav_rollout_metrics}, the temporal baseline performs worst, while WorldMem and LongLive-RAG benefit from geometry- and appearance-based retrieval.
At loop closure, our learned visual retriever outperforms LongLive-RAG ($17\%$ lower DreamSim), and our learned metadata retriever outperforms WorldMem ($19\%$ lower DreamSim), despite using the same respective cue types.
\Cref{fig:teaser} illustrates why: WorldMem retrieves a frame with high FOV overlap but a heavily occluded view of the goal, whereas our learned metadata retriever selects a frame with a clearer predictive view.
This highlights a key limitation of hand-crafted relevance rules: cue similarity does not necessarily imply predictive utility.
Finally, the \emph{Multi-Cue} variant in \Cref{fig:loopnav_rollout_metrics}, which fuses metadata and visual cues, shows that FAR can exploit complementary retrieval signals.
Its rollout error closely follows the better of the metadata-only and visual-only variants across the horizon, suggesting that adaptive fusion emphasizes whichever cue is more informative for the current prediction.

\begin{figure}[t]
\vspace{-3mm}
\centering
\includegraphics[width=1.0\linewidth]{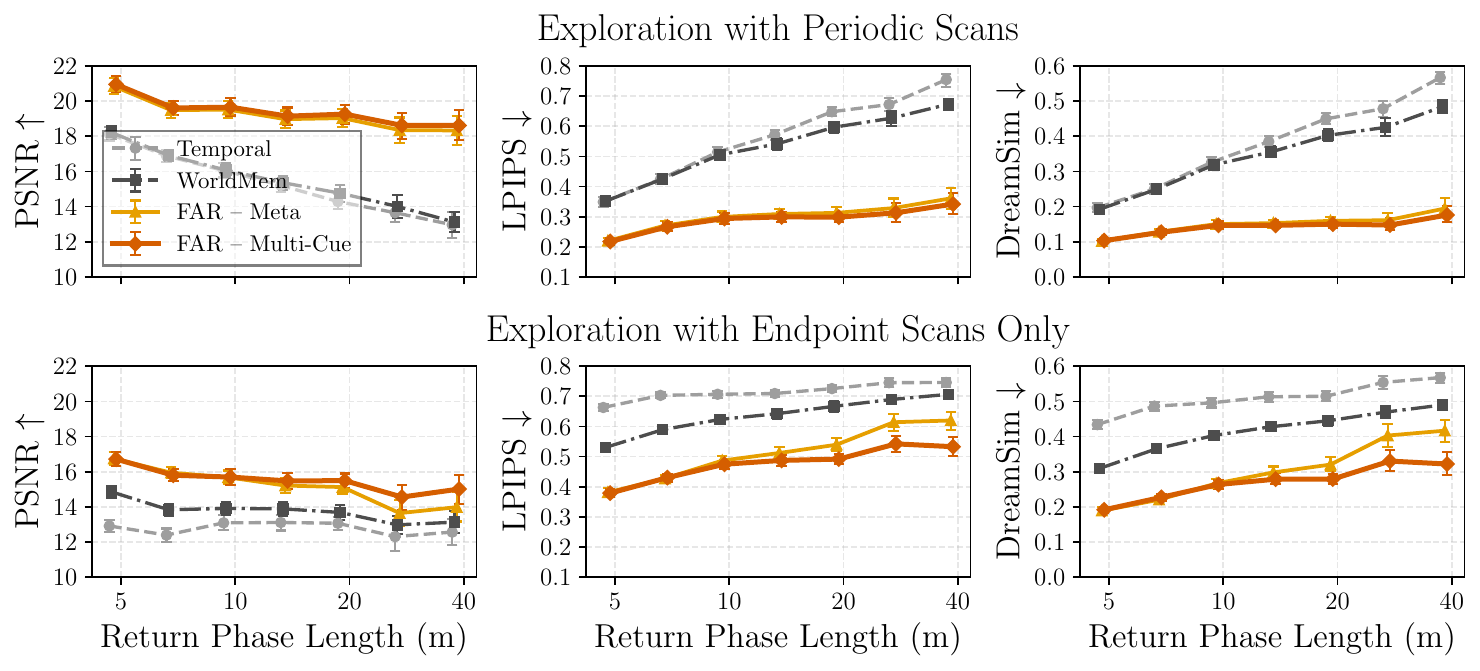}
\vspace{-6mm}
\caption{\textbf{SoundSpaces rollout quality.} We report PSNR, LPIPS, and DreamSim, where the error is averaged over predicted frames along return trajectory. Thus, the $x$-axis denotes the length of the trajectory over which the rollout is evaluated, \emph{not the rollout horizon}. \textbf{Top:} results on the corpus where the agent performs periodic \(360^\circ\) scans during exploration. \textbf{Bottom:} results on the corpus where the agent scans at the endpoints, i.e., at the beginning and end of exploration.}
\vspace{-4mm}
\label{fig:ss_rollout_metrics}
\end{figure}

\begin{figure}[t]
\centering
\includegraphics[width=1.0\linewidth]{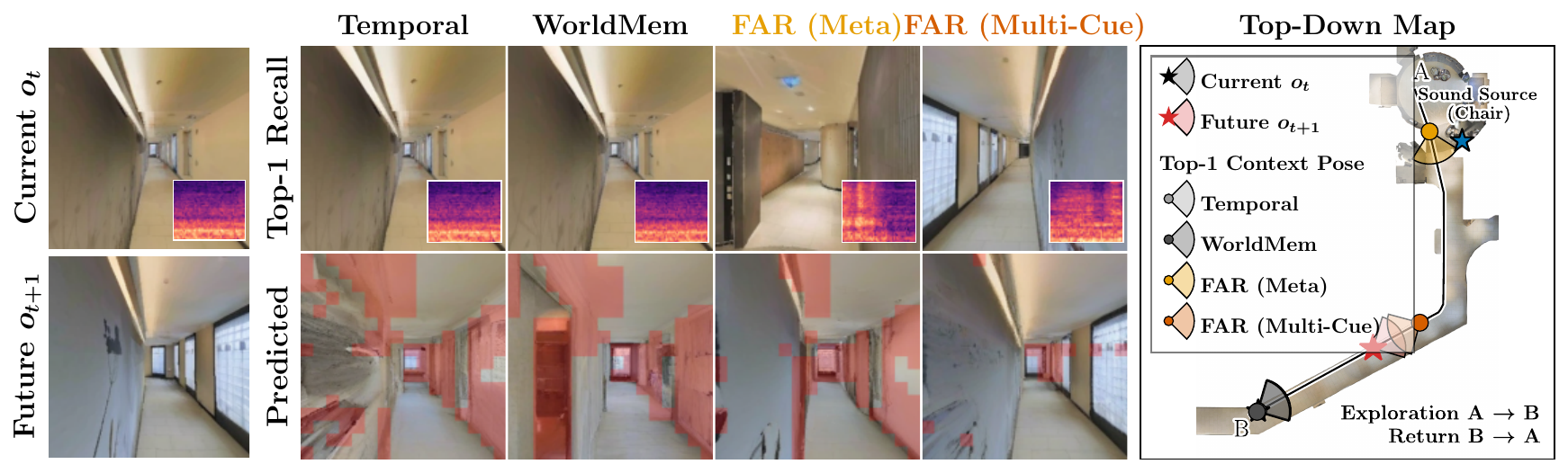}
\vspace{-6mm}
\caption{\textbf{Retrieval comparison on SoundSpaces.} Given the observation and goal view, we visualize the top-1 context retrieved by each method and the generated goal frame, with regions of large error overlaid in \textcolor{red}{\textbf{red}}. The spectrograms show the audio cue associated with each state. The map shows the exploration and return trajectories with the memories selected by each method.}
\vspace{-4mm}
\label{fig:ss_retrieval}
\end{figure}

\vspace{-1mm}
\subsection{SoundSpaces -- FAR Adaptively Learns Which Cues to Trust}

We next evaluate FAR on SoundSpaces, where we construct an indoor navigation dataset augmented with spatial audio.
We generate two corpora that differ in the density of available episodic memories: one performs periodic \(360^\circ\) scans during exploration, while the other scans only at the exploration endpoints.
The former thus provides substantially denser historical coverage from which the world model can retrieve.
We stratify evaluation by return-phase length, i.e., the physical distance traversed during the trajectory over which rollout quality is measured.
Longer return trajectories are generally more challenging, as they require prediction through more rooms, corridors, and viewpoint changes.

\Cref{fig:ss_rollout_metrics} shows that FAR consistently outperforms the temporal baseline and WorldMem across both corpora.
With periodic scans, metadata alone provides a strong retrieval signal, and Multi-Cue fusion of metadata and audio yields modest gains.
Under the sparser endpoint-scan setting, however, the advantage of Multi-Cue grows with return-phase length, particularly in LPIPS and DreamSim.
This suggests that metadata often suffices when relevant visual memories are densely available, whereas audio provides a valuable complementary cue for long trajectories or sparse histories.

\Cref{fig:ss_retrieval} illustrates how audio can complement metadata when geometric relevance alone is insufficient.
The agent is traversing a corridor and must predict a view farther down.
Temporal retrieval selects the most recent observation, which looks in the correct direction but provides limited information about the distant portion of the corridor.
WorldMem selects the same context because it has high FOV overlap with the goal pose, and consequently suffers from the limitation.
FAR with metadata retrieves an observation taken farther along the trajectory and oriented toward the goal, but the view is obstructed by the corridor geometry.
Incorporating audio cue instead retrieves a context with a clearer view of the region around the goal, leading to a more faithful prediction.
The inset spectrograms show the audio associated with each retrieved state, highlighting how spatial audio can provide a complementary retrieval cue when pose and visual geometry alone are ambiguous.

\begin{figure}[t]
\vspace{-2mm}
\centering
\includegraphics[width=1.0\linewidth]{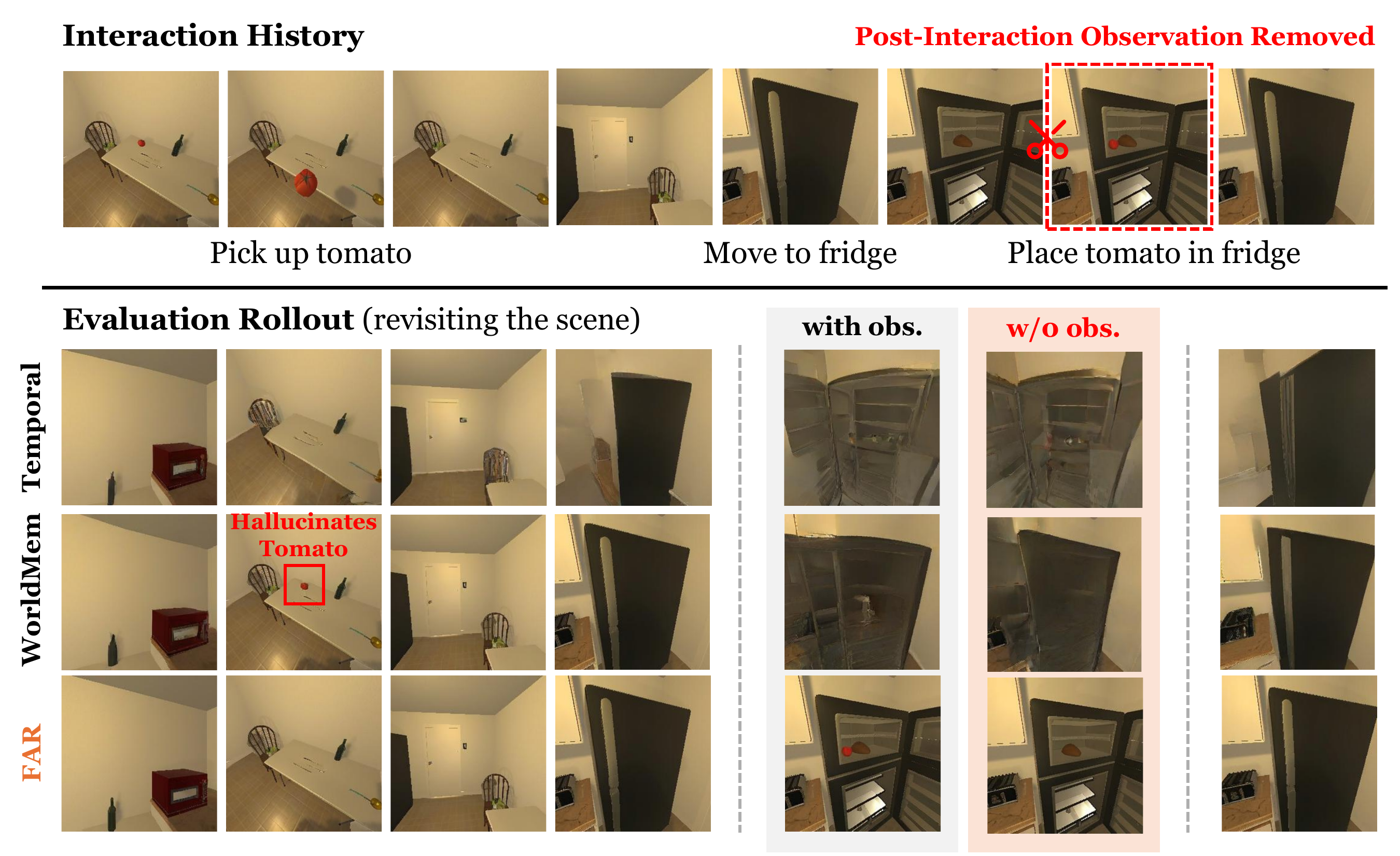}
\vspace{-6mm}
\caption{\textbf{History-dependent counterfactual prediction in AI2-THOR.}
We compare two trajectories with the same current closed-fridge observation and opening action but different histories: one history contains the event of placing a tomato inside the fridge, while the other does not.}
\vspace{-4mm}
\label{fig:ai2thor_counter}
\end{figure}

\vspace{-2mm}
\subsection{AI2-THOR -- FAR Recalls the Right State as the World Changes}

\begin{wrapfigure}{r}{0.35\linewidth}
\centering
\vspace{-6mm}
\includegraphics[width=0.9\linewidth]{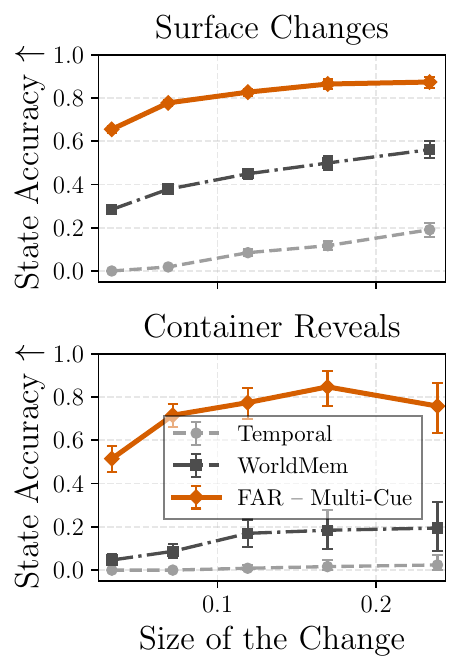}
\vspace{-4mm}
\caption{\textbf{State accuracy under object manipulation.}}
\label{fig:ai2thor_acc}
\vspace{-6mm}
\end{wrapfigure}

We next evaluate FAR in AI2-THOR environments where interactions change the state of the world.
During exploration, the agent randomly moves objects between surfaces and containers, producing multiple observations of the same location that may correspond to conflicting world states.
During the return phase, the agent revisits these locations and must generate observations consistent with the state established by the preceding interactions.
This setting is challenging for retrieval based primarily on recency or geometry, since a spatially well-matched memory may depict a stale state.

\textbf{State rendering accuracy.}
We evaluate whether the generated rollout reflects the current world state rather
than a stale state from earlier in the trajectory.
Specifically, we report the fraction of changed states satisfying
$d(\bmo_{\mathrm{GT}},\bmo_{\mathrm{Pred}})
<
d(\bmo_{\mathrm{GT}},\bmo_{\mathrm{Stale}})$,
where $\bmo_{\mathrm{GT}}$ is the ground-truth frame,
$\bmo_{\mathrm{Pred}}$ is the predicted rollout frame,
$\bmo_{\mathrm{Stale}}$ is a pose-matched observation depicting the
pre-interaction state, and $d$ denotes LPIPS on a crop around the changed region.

\Cref{fig:ai2thor_acc}~stratifies accuracy by state-change magnitude,
$d(\bmo_{\mathrm{GT}},\bmo_{\mathrm{Stale}})$.
We consider two cases:
\emph{surface changes}, where the updated state is directly visible upon revisit,
and \emph{container reveals}, where it becomes observable only after opening a
container.
Multi-Cue FAR, fusing metadata and vision, substantially outperforms temporal
and geometry-based recall across change magnitudes in both settings.
The advantage is especially pronounced for container reveals, where the current
observation contains little information about hidden contents and accurate
prediction therefore depends strongly on recalling the relevant interaction
history.
Performance decreases for all methods when current and stale states differ only
locally, consistent with the limited sensitivity of the diffusion prediction
objective to small visual changes.

\textbf{History-dependent counterfactual prediction.}
\Cref{fig:ai2thor_counter} shows a paired counterfactual example in which the current observation and action are held fixed, while the preceding interaction history differs.
In one history, the agent places a tomato inside the refrigerator; in the other, it does not.
When the agent later observes the same closed refrigerator and executes the same opening action, FAR generates different futures consistent with the corresponding histories, rendering the tomato inside the refrigerator only when it was previously placed there.
This shows that the predicted future is not determined by the current observation alone, but depends on recalling the relevant past interaction.
In contrast, Temporal and WorldMem fail to preserve the interaction-dependent state, with WorldMem additionally hallucinating the tomato at its stale table location as a result of prioritizing stale contexts with high FOV overlap over up-to-date contexts.

\begin{figure}[t]
\vspace{-1mm}
\centering
\includegraphics[width=0.3\linewidth]{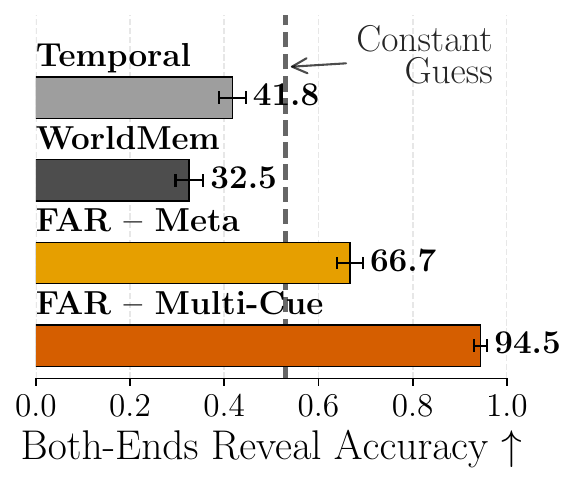} \hfill
\includegraphics[width=0.68\linewidth]{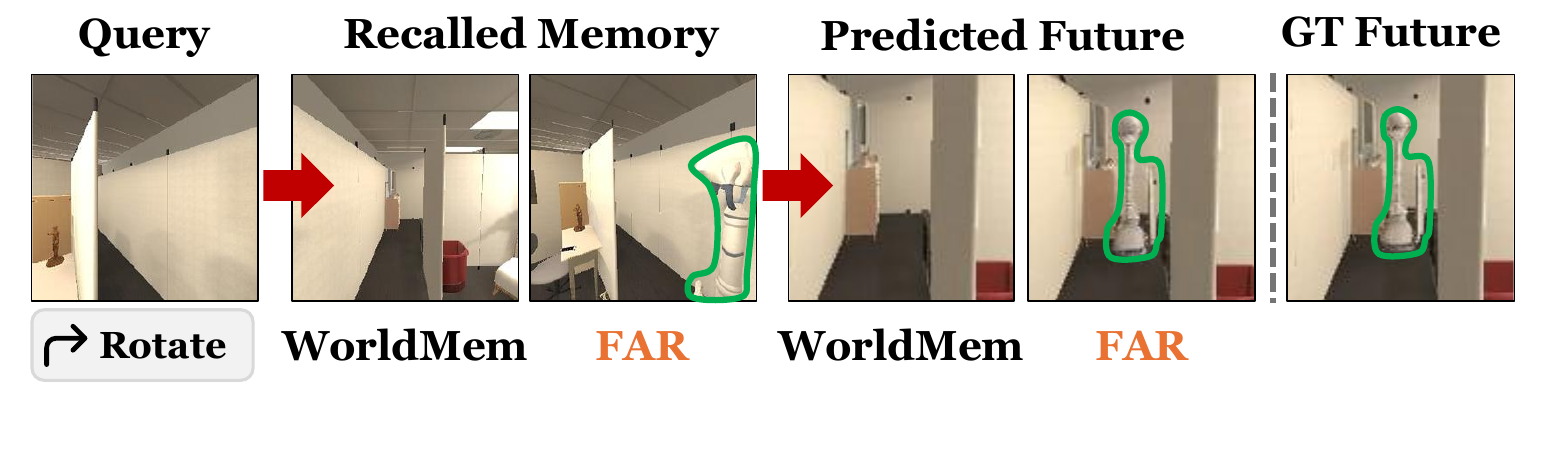}
\vspace{-2mm}
\caption{\textbf{Off-scene dynamics prediction.}
\textbf{Left:} trajectory-level accuracy for predicting the future location of the independently moving agent $A_2$.
\textbf{Right:} qualitative example showing that FAR recalls a motion-informative episode containing $A_2$, and predicts the correct future state with $A_2$. WorldMem retrieves a spatially relevant but dynamically uninformative memory, failing to render $A_2$. All $A_2$ appearances are outlined in {\color{ForestGreen}\textbf{green}} for visual aid.}
\vspace{-4mm}
\label{fig:ai2thor_dyn}
\end{figure}

\textbf{Off-scene dynamics prediction.}
We test whether FAR can predict off-scene dynamics in a two-agent corridor setting.
$A_1$ is the observer, while $A_2$ patrols back and forth and is only intermittently visible through a doorway.
The task is to predict $A_2$'s location when $A_1$ exits the room and looks down the corridor.
To capture motion, each recalled context consists of two observations separated by $1.5$ seconds.
We compare FAR using only time and pose metadata with a \emph{Multi-Cue} variant that fuses metadata with an $A_2$ observation cue.
As shown in \Cref{fig:ai2thor_dyn}, Temporal and WorldMem achieve $41.9\%$ and $32.4\%$ accuracy, while FAR with metadata alone reaches $67.1\%$, indicating that future-aware training can learn when informative crossings are likely to occur.
Multi-Cue further improves accuracy to $94.6\%$ by more precisely identifying motion-informative episodes.
Qualitatively, WorldMem retrieves a geometrically relevant but dynamically uninformative memory, whereas Multi-Cue FAR recalls an episode containing $A_2$ and predicts its future off-scene location correctly.

\vspace{-2mm}
\subsection{Ablation Study}
\vspace{-1mm}

\begin{wraptable}{r}{0.58\linewidth}
\begin{minipage}{0.58\textwidth}
\vspace{-4mm}
\centering
\resizebox{1.0\textwidth}{!}{
\begin{tabular}{lccc}
\toprule
\textbf{Retrieval} & \textbf{PSNR}$\uparrow$ & \textbf{LPIPS}$\downarrow$ & \textbf{DreamSim}$\downarrow$ \\
\cmidrule{1-4}
Temporal (NWM) & $13.229_{\pm 0.114}$ & $0.582_{\pm 0.007}$ & $0.200_{\pm 0.004}$ \\
LongLive-RAG   & $14.847_{\pm 0.121}$ & $0.463_{\pm 0.006}$ & $0.133_{\pm 0.003}$ \\
WorldMem       & $15.372_{\pm 0.112}$ & $0.448_{\pm 0.006}$ & $	0.130_{\pm 0.003}$ \\
\cmidrule{1-4}
\multicolumn{4}{l}{\textbf{{FAR} with Vision Cue}} \\
Enc. Pretrained           & $16.248_{\pm 0.112}$ & $0.390_{\pm 0.006}$ & $0.100_{\pm 0.003}$ \\
+ MLP Adapter             & $16.480_{\pm 0.108}$ & $0.374_{\pm 0.006}$ & $0.091_{\pm 0.002}$ \\
\quad + Meta, $\lambda=0.5$ & $16.912_{\pm 0.117}$ & $0.353_{\pm 0.005}$ & $0.088_{\pm 0.002}$ \\
\quad + Meta, $\lambda$ learned & $\mathbf{17.267_{\pm 0.116}}$ & $\mathbf{0.337_{\pm 0.005}}$ & $\mathbf{0.082_{\pm 0.002}}$ \\
\bottomrule
\end{tabular}}
\vspace{-2mm}
\caption{\textbf{Ablation on LoopNav.} We report average loop closure errors for various configurations.}
\label{tab:loopnav_ablation}
\vspace{-2.5mm}
\end{minipage}
\end{wraptable}

\Cref{tab:loopnav_ablation} analyzes each component of FAR.
A frozen pretrained vision encoder provides a strong retrieval signal, substantially outperforming baselines across three rollout metrics.
Adapting the query representation with MLP further improves performance from $16.25$ to $16.48$ PSNR, with gains in LPIPS and DreamSim, indicating that visual similarity benefits from task-specific adaptation toward predictive relevance.
Incorporating metadata cues yields a larger improvement: fixed fusion with $\lambda=0.5$ increases PSNR to $16.91$.
Finally, learning the fusion weights achieves the best performance across all metrics, showing that metadata and visual cues are complementary and adaptively weighting their reliability is preferable to uniform fusion.

\vspace{-2mm}
\section{Conclusion and Limitations}
\vspace{-1mm}

We introduced FAR, which learns which episodic
memories support prediction and which retrieval cues to trust.
FAR uses the observed future for predictive credit during training, while
remaining future-blind at inference.
By learning cue-specific relevance and query-dependent cue weights, FAR improves
long-horizon prediction across navigation, multi-cue, and changing-state
settings, supporting predictive utility as a principled basis for episodic
memory recall in world models.

Our work focuses on recall from an external episodic memory, leaving memory
writing, compression, forgetting, and higher-order interactions among recalled
memories to future work.
FAR also requires informative retrieval cues and additional training-time
computation for predictive-utility evaluation.
Our diffusion-loss utility may underweight semantically important local
changes; alternative task-aware utility functions are a promising direction.
Finally, our experiments use controlled simulations, motivating evaluation in
richer real-world settings and integration with learned memory formation and
persistent-state representations.

\section*{Acknowledgments}

We sincerely thank Koichi Saito for his thoughtful insights and engaging discussions throughout this project.

\bibliography{iclr2027_conference}
\bibliographystyle{iclr2027_conference}

\clearpage
\newpage

\appendix

\begin{center}
    \textbf{\Large Supplementary Material}
\end{center}

\section{Additional Discussion of Related Work} 
\label{app:related_work}

\begin{table}[!ht]
\centering
\setlength{\tabcolsep}{3.5pt}
\renewcommand{\arraystretch}{0.95}
\resizebox{\columnwidth}{!}{
\begin{tabular}{lccccc}
\toprule
\textbf{Method} &
\makecell{\textbf{World}\\\textbf{Model}} &
\makecell{\textbf{External}\\\textbf{Retriever}} &
\makecell{\textbf{Prediction}\\\textbf{Supervised}} &
\makecell{\textbf{Adaptive}\\\textbf{Cue Fusion}} &
\makecell{\textbf{Cue Options}} \\
\midrule
FramePack~\citep{zhang2025framepack} & \xmark & \cmark & \xmark & \xmark & Time, Vision \\
LongLive-RAG~\citep{hu2026longliverag} & \xmark & \cmark & \xmark & \xmark & Vision \\
WorldMem~\citep{xiao2025worldmem} & \cmark & \cmark & \xmark & \xmark & Time, Pose \\
Context-as-Memory~\citep{yu2025cam} & \cmark & \cmark & \xmark & \xmark & Pose \\
VRAG~\citep{chen2025vrag} & \cmark & \cmark & \xmark & \xmark & Pose \\
SPMem~\citep{wu2025spmem} & \cmark & \cmark & \xmark & \xmark & Pose \\
I3DM~\citep{li2026i3dm} & \cmark & \cmark & \xmark & \xmark & Pose, Vision \\
Matrix-Game 3.0~\citep{wang2026matrixgame3} & \cmark & \cmark & \xmark & \xmark & Pose \\
\midrule
Mixture-of-Contexts~\citep{cai2026moc} & \xmark & \xmark & \cmark & \xmark & Vision, Text \\
MemLearner~\citep{yu2026memlearner} & \cmark & \xmark & \cmark & \xmark & Pose, Vision \\
Compression-and-Retrieval~\citep{peng2026car} & \cmark & \xmark & \cmark & \xmark & Pose, Vision \\
\midrule
\rowcolor{black!6}
\textbf{Ours} & \cmark & \cmark & \cmark & \cmark & Time, Pose, Vision, Audio \\
\bottomrule
\end{tabular}}
\caption{\textbf{Conceptual comparison of episodic memory recall for video and world models.} \textbf{World Model} indicates whether the method is designed for action-conditioned future prediction; \textbf{External Retriever} whether memory selection is performed outside the generator; \textbf{Prediction Supervised} whether retrieval is optimized using the downstream prediction objective; \textbf{Adaptive Cue Fusion} whether the retriever dynamically weights multiple retrieval cues for each query rather than using a fixed cue or predefined combination; and \textbf{Cue Options} lists the information sources used to determine memory relevance in the respective papers.}
\label{tab:related_memory_retrieval}
\end{table}

\paragraph{Episodic memory retrieval for world models.}
Recent video and world models augment limited context windows with external memories that preserve historical frames or latent representations for later retrieval~\citep{xiao2025worldmem,yu2025cam,wu2025spmem,li2026i3dm,hu2026longliverag}.
As summarized in \Cref{tab:related_memory_retrieval}, existing external retrievers typically rely on fixed notions of relevance, including temporal proximity, pose or field-of-view overlap, geometry-aware similarity, or learned visual embeddings.
These approaches demonstrate that selectively revisiting non-local observations can improve long-horizon consistency, but the retrieval rule itself is generally not optimized using the downstream world-model prediction objective.
FAR instead learns memory relevance directly from predictive utility and further learns query-dependent fusion over heterogeneous temporal, spatial, visual, and auditory cues.
Thus, FAR learns both \emph{which memories are predictively useful} and \emph{which cues should be trusted to find them}.

\paragraph{Internal long-context memory.}
A complementary line retains long-range history within the generator and learns how to route or compress it during prediction.
Mixture of Contexts sparsely routes queries to informative history chunks, MemLearner queries context using the video generator itself, and CaR combines attention-based memory retrieval with context compression~\citep{cai2026moc,yu2026memlearner,peng2026car}.
Recent world models also use recurrent or hybrid linear attention to propagate information over long sequences, including L2World and SANA-WM~\citep{wang2026context,zhu2026sanawm}.
These internal mechanisms and external episodic recall are complementary: the former efficiently summarize or route long histories within the model, while the latter preserves individually addressable observations that can be selectively revisited.

\paragraph{Persistent-state representations.}
Rather than retrieving individual observations, another family of approaches maintains an explicit or latent state that is updated as the environment evolves.
PERSIST evolves a latent 3D scene representation, Spatia maintains an updatable 3D point-cloud memory, and EgoSim continuously updates a 3D scene state across embodied interactions~\citep{garcin2026persist,zhao2026spatia,hao2026egosim}.
Persistent state provides a compact, viewpoint-independent representation of geometry and world changes, whereas episodic memory preserves past observations as individually recoverable evidence.
The two are complementary: persistent state summarizes the model's current belief about the world, while episodic recall can recover specific past evidence when that state is incomplete or uncertain.

\paragraph{Learned retrieval and retrieval-augmented generation.}
Retrieval-augmented generation (RAG) treats retrieved documents as latent context and learns retrieval jointly with downstream generation~\citep{lewis2020rag,sachan2021emdr}.
In particular, \emdr~uses the reader likelihood of the observed target to provide posterior supervision for discrete document retrieval~\citep{sachan2021emdr}.
FAR draws on this principle but applies it to an agent's own episodic history: the realized future provides predictive credit for candidate memories, and this supervision trains a retriever that operates without access to the future at inference.
Unlike document RAG, FAR additionally learns to combine heterogeneous temporal, spatial, visual, and auditory cues to locate useful memories.

\paragraph{Long-horizon world models.}
A broader line of work targets long-horizon consistency by changing how context or state is represented rather than explicitly learning episodic recall.
FramePack compresses frame histories according to importance and introduces anti-drifting mechanisms, while MilliVid uses hierarchical latents and coarse-to-fine rollout to preserve long-range structure efficiently~\citep{zhang2025framepack,preetam2026millivid}.
Seoul World Model instead repeatedly re-grounds generation using nearby real-world observations to stabilize long trajectories~\citep{seo2026swm}.
These approaches address complementary challenges of context scaling, drift, and external grounding; FAR focuses specifically on learning which observations from the agent's own history are useful for the current prediction.

\section{Additional Theoretical Details}

In this section, we provide more theoretical and practical details about FAR. First, we provide the proof of \Cref{prop:relevance-bound} in \Cref{app:predictive_utility_derivation}. 
We summarize the overall pseudo code as \Cref{alg:joint_training}.

\subsection{Derivation of Predictive Relevance Bound}
\label{app:predictive_utility_derivation}

\begin{proof}
Consider the joint distribution induced by the environment data distribution and the retriever,
\begin{align*}
p_\phi(\bmo_{t+1},\cM_{t-1},\cZ_{t-1},\cR_t,\cQ_t,\cC_t)
\coloneqq
p_{\rm data}(\bmo_{t+1},\cM_{t-1},\cZ_{t-1},\cR_t,\cQ_t)\,
r_\phi(\cC_t\mid\cZ_{t-1},\cR_t).
\end{align*}
The subscript $\phi$ reflects that the recalled context $\cC_t$ is selected by the retriever $r_\phi$.

Introducing the retriever does not change the marginal distribution of the future and prediction query.
Because $r_\phi(\cdot\mid\cZ_{t-1},\cR_t)$ is normalized,
\begin{align*}
p_\phi(\bmo_{t+1},\cQ_t) 
&=
\int
p_{\rm data}(\bmo_{t+1},\cM_{t-1},\cZ_{t-1},\cR_t,\cQ_t)
\sum_{\cC_t}
r_\phi(\cC_t\mid\cZ_{t-1},\cR_t)
\,\mathrm{d}\cM_{t-1}\,\mathrm{d}\cZ_{t-1}\,\mathrm{d}\cR_t
\\
&=
p_{\rm data}(\bmo_{t+1},\cQ_t).
\end{align*}
Hence,
\begin{align*}
p_\phi(\bmo_{t+1}\mid\cQ_t)
=
p_{\rm data}(\bmo_{t+1}\mid\cQ_t).
\end{align*}

By the definition of conditional mutual information,
\begin{align*}
I_\phi(\bmo_{t+1};\cC_t\mid\cQ_t)
=
\bbE_{p_\phi}
\left[
\log
\frac{
p_\phi(\bmo_{t+1}\mid\cC_t,\cQ_t)
}{
p_\phi(\bmo_{t+1}\mid\cQ_t)
}
\right]
=
\bbE_{p_\phi}
\left[
\log p_\phi(\bmo_{t+1}\mid\cC_t,\cQ_t)
-
\log p_{\rm data}(\bmo_{t+1}\mid\cQ_t)
\right].
\end{align*}

Introducing any predictive distribution $p_\theta(\bmo_{t+1}\mid\cC_t,\cQ_t)$ and adding and subtracting its log-density gives
\begin{align*}
I_\phi(\bmo_{t+1};\cC_t\mid\cQ_t)
&=
\bbE_{p_\phi}
\left[
\log p_\theta(\bmo_{t+1}\mid\cC_t,\cQ_t)
-
\log p_{\rm data}(\bmo_{t+1}\mid\cQ_t)
\right]
\\
&\quad+
\bbE_{p_\phi(\cC_t,\cQ_t)}
\left[
\KL\!\left(
p_\phi(\bmo_{t+1}\mid\cC_t,\cQ_t)
\,\middle\|\,
p_\theta(\bmo_{t+1}\mid\cC_t,\cQ_t)
\right)
\right].
\end{align*}

The second term is nonnegative by the nonnegativity of KL divergence.
Therefore,
\begin{align*}
I_\phi(\bmo_{t+1};\cC_t\mid\cQ_t)
\geq
\bbE_{\substack{
(\bmo_{t+1},\cM_{t-1},\cZ_{t-1},\cR_t,\cQ_t)\sim p_{\rm data}\\
\cC_t\sim r_\phi(\cdot\mid\cZ_{t-1},\cR_t)
}}
\left[
\log p_\theta(\bmo_{t+1}\mid\cC_t,\cQ_t)
-
\log p_{\rm data}(\bmo_{t+1}\mid\cQ_t)
\right],
\end{align*}
which is exactly \Cref{eq:predictive_information_bound}.

Moreover, the decomposition above shows that the gap in the bound is
\begin{align*}
\bbE_{p_\phi(\cC_t,\cQ_t)}
\left[
\KL\!\left(
p_\phi(\bmo_{t+1}\mid\cC_t,\cQ_t)
\,\middle\|\,
p_\theta(\bmo_{t+1}\mid\cC_t,\cQ_t)
\right)
\right].
\end{align*}
Thus, the bound is tight when the predictive world model matches the conditional distribution of the future induced by the data distribution and retriever, almost surely over $(\cC_t,\cQ_t)$.
\end{proof}

\subsection{Latent-Context Maximum-Likelihood View of the Retriever Update}
\label{app:retriever_objective}

We show that the future-aware posterior and retriever objective in
\Cref{eq:future_posterior,eq:posterior_distillation} follow directly from
maximum likelihood in the latent-context predictive model
\Cref{eq:latent_context_model}.
This also clarifies why FAR matches a posterior that combines predictive utility
with the current recall distribution, rather than a target based on predictive
utility alone.

For readability, within this subsection we abbreviate
\begin{align*}
r_\phi(\cC_t)
\coloneqq
r_\phi(\cC_t\mid\cZ_{t-1},\cR_t),\quad\text{and}\quad
u_\theta(\cC_t)
\coloneqq
u_\theta(\cC_t\mid\bmo_{t+1},\cQ_t),
\end{align*}
and let
\begin{align*}
\mathfrak C_t
\coloneqq
\left\{
\cC_t\subseteq\cM_{t-1}
:
|\cC_t|=K
\right\}
\end{align*}
denote the set of all valid recalled context sets, i.e., all subsets of the
episodic memory containing exactly $K$ historical observations.
For a fixed training example, \Cref{eq:latent_context_model} gives
\begin{align*}
p_{\theta,\phi}(\bmo_{t+1}\mid\cdots)
=
\sum_{\cC_t\in\mathfrak C_t}
r_\phi(\cC_t)\,
p_\theta(\bmo_{t+1}\mid\cC_t,\cQ_t)=
\sum_{\cC_t\in\mathfrak C_t}
r_\phi(\cC_t)\,
\exp u_\theta(\cC_t).
\end{align*}

The corresponding posterior over recalled contexts is
\begin{align*}
q_{\theta,\phi}(\cC_t)
\coloneqq
\frac{
r_\phi(\cC_t)\exp u_\theta(\cC_t)
}{
\sum_{\cC_t'\in\mathfrak C_t}
r_\phi(\cC_t')\exp u_\theta(\cC_t')
}.
\end{align*}
For any distribution $q$ over $\mathfrak C_t$, a direct rearrangement gives
the standard latent-variable decomposition
\begin{align}
\log p_{\theta,\phi}(\bmo_{t+1}\mid\cdots)
&=
\bbE_{\cC_t\sim q}
\left[
u_\theta(\cC_t)
\right]
-
\KL\!\left(
q
\,\middle\|\,
r_\phi
\right)
+
\KL\!\left(
q
\,\middle\|\,
q_{\theta,\phi}
\right).
\label{eq:latent_context_variational}
\end{align}
Since the final term is nonnegative, the bound is maximized at
$q=q_{\theta,\phi}$, or equivalently,
\begin{align*}
q_{\theta,\phi}
=
\underset{q}{\arg\max}\;
\left\{
\bbE_{\cC_t\sim q}
\left[
u_\theta(\cC_t)
\right]
-
\KL\!\left(
q
\,\middle\|\,
r_\phi
\right)
\right\}.
\end{align*}
Thus, the future-aware posterior increases expected predictive utility while
remaining anchored to the relevance inferred by the future-blind retriever.

Because
$r_\phi(\cC_t)\propto
\exp s_\phi(\cC_t\mid\cZ_{t-1},\cR_t)$,
the optimizer has the form
\begin{align*}
q_{\theta,\phi}(\cC_t)
\propto
\exp\!\left[
s_\phi(\cC_t\mid\cZ_{t-1},\cR_t)
+
u_\theta(\cC_t\mid\bmo_{t+1},\cQ_t)
\right],
\end{align*}
which recovers the future-aware posterior in
\Cref{eq:future_posterior}.
The realized future therefore updates the relevance already inferred from the
available retrieval cues, rather than replacing it.

This also explains why we do not construct the target from predictive utility
alone.
Normalizing only predictive utility would give
\begin{align*}
\bar q_\theta(\cC_t)
\propto
\exp u_\theta(\cC_t).
\end{align*}
For finite $\mathfrak C_t$, this distribution solves
\begin{align*}
\bar q_\theta
=
\underset{q}{\arg\max}\;
\left\{
\bbE_q[u_\theta]
+
\mathcal H(q)
\right\}=
\underset{q}{\arg\max}\;
\left\{
\bbE_q[u_\theta]
-
\KL\!\left(
q
\,\middle\|\,
\operatorname{Unif}(\mathfrak C_t)
\right)
\right\},
\end{align*}
up to the constant $\log|\mathfrak C_t|$.
Hence, a utility-only target implicitly replaces the query-dependent recall
distribution $r_\phi$ with a uniform prior over contexts, discarding the
relevance already inferred from the current query and retrieval cues.
For example, if all contexts have equal predictive utility, then
$q_{\theta,\phi}=r_\phi$, so the realized future provides no evidence for
changing the current recall preference.
In contrast, $\bar q_\theta$ becomes uniform and would alter the retriever
despite receiving no differential predictive evidence.
More generally, $q_{\theta,\phi}\propto r_\phi\exp u_\theta$ is the posterior
implied by the latent-context model, whereas $\bar q_\theta$ corresponds to a
different model with a uniform latent-context prior.

The same variational formulation determines how the retriever should be
updated.
Let $\phi^{-}$ denote the current retriever parameters, and first compute
$q_{\theta,\phi^{-}}$ using the current model.
Holding this posterior fixed, maximizing
\Cref{eq:latent_context_variational} with respect to $\phi$ reduces to
\begin{align*}
\underset{\phi}{\max}\;
\bbE_{\cC_t\sim q_{\theta,\phi^{-}}}
\left[
\log r_\phi(\cC_t)
\right],
\end{align*}
or equivalently,
\begin{align*}
\underset{\phi}{\min}\;
\KL\!\left(
q_{\theta,\phi^{-}}
\,\middle\|\,
r_\phi
\right).
\end{align*}
Thus, the update alternates between using the realized future to compute
posterior credit and fitting the future-blind retriever to this fixed target.
In the notation of the main text, this is implemented by
\begin{align*}
\cL_{\ret}(\phi)
\coloneqq
\bbE
\left[
\KL\!\left(
\sg[q_{\theta,\phi}]
\,\middle\|\,
r_\phi
\right)
\right],
\end{align*}
which is exactly \Cref{eq:posterior_distillation}.
The stop-gradient prevents gradients from flowing through the posterior during
the retriever update, implementing the fixed-target step of this alternating
optimization rather than introducing an additional heuristic objective.

With the exact posterior, the same update also recovers the retriever-side
marginal-likelihood gradient.
At the parameters $\phi=\phi^{-}$ used to construct the posterior,
\begin{align}
\nabla_\phi
\log p_{\theta,\phi}(\bmo_{t+1}\mid\cdots)
\big|_{\phi=\phi^{-}}
=
\bbE_{\cC_t\sim q_{\theta,\phi^{-}}}
\left[
\nabla_\phi
\log r_\phi(\cC_t)
\right]_{\phi=\phi^{-}}
=
-
\nabla_\phi
\KL\!\left(
q_{\theta,\phi^{-}}
\,\middle\|\,
r_\phi
\right)
\bigg|_{\phi=\phi^{-}}.
\label{eq:retriever_gradient_identity}
\end{align}
Hence, for the exact posterior, posterior matching gives exactly the negative
marginal-likelihood gradient with respect to the retriever.

The derivation above applies to the exact posterior over recalled context sets.
Under discrete Top-$K$ recall, enumerating these sets is combinatorial.
The candidate-level objective in
\Cref{app:retriever_posterior_approx} therefore applies the same posterior-credit
principle to a finite candidate pool using singleton predictive utilities.
Because this restricts both the context space and the utility evaluation, the
exact gradient identity in \Cref{eq:retriever_gradient_identity} need not hold
globally for the practical approximation.

\subsection{Retriever Posterior Approximation}
\label{app:retriever_posterior_approx}

The future-aware posterior in \Cref{eq:future_posterior} is defined over all
admissible recalled context sets.
Under discrete Top-$K$ recall, exact evaluation is therefore combinatorial.
We adapt the latent-variable retriever training strategy of
\emdr~\citep{sachan2021emdr} to obtain a tractable candidate-level approximation
over a finite pool of historical observations.

For each retriever update, let
$\mathcal A_t\subseteq\cM_{t-1}$ denote the candidate pool.
The observation-level relevance scores from
\Cref{eq:multi_cue_score} induce the future-blind candidate distribution
\begin{align}
r_\phi^{\mathcal A_t}
\!\left(
\bmo_i
\mid
\cZ_{t-1},
\cR_t
\right)
\coloneqq
\operatorname{softmax}_{\bmo_i\in\mathcal A_t}
s_{\phi,i}.
\label{eq:candidate_prior}
\end{align}
This distribution represents the retriever's relative preference among the
candidates using only information available at inference time.

To assign predictive credit, we evaluate how well each candidate individually
supports prediction of the observed future.
The singleton predictive utility of candidate $\bmo_i$ is
\begin{align}
u_{\theta,i}
\coloneqq
u_\theta
\!\left(
\{\bmo_i\}
\mid
\bmo_{t+1},
\cQ_t
\right)
=
\log
p_\theta
\!\left(
\bmo_{t+1}
\mid
\{\bmo_i\},
\cQ_t
\right).
\label{eq:singleton_utility}
\end{align}
Applying the same predictive reweighting principle as
\Cref{eq:future_posterior} gives the future-aware candidate posterior
\begin{align}
q_{\theta,\phi}^{\mathcal A_t}
\!\left(
\bmo_i
\mid
\bmo_{t+1},
\cZ_{t-1},
\cR_t,
\cQ_t
\right)
\coloneqq
\operatorname{softmax}_{\bmo_i\in\mathcal A_t}
\left[
s_{\phi,i}
+
u_{\theta,i}
\right].
\label{eq:candidate_posterior}
\end{align}
Thus, $r_\phi^{\mathcal A_t}$ captures which candidates appear relevant before
the future is known, while $q_{\theta,\phi}^{\mathcal A_t}$ reweights them
according to their predictive utility for the realized future.

We train the retriever by matching its future-blind candidate distribution to
this future-aware target:
\begin{align}
\widehat{\cL}_{\ret}(\phi)
\coloneqq
\bbE
\left[
\KL\!\left(
\sg\!\left[
q_{\theta,\phi}^{\mathcal A_t}
\right]
\,\middle\|\,
r_\phi^{\mathcal A_t}
\right)
\right],
\label{eq:candidate_distillation}
\end{align}
where $\sg[\cdot]$ denotes stop-gradient and the expectation is over training
examples and candidate-pool construction.
For a fixed candidate pool with exact singleton likelihoods, this update yields
the same retriever gradient as the corresponding candidate-level latent-variable
objective of \emdr~\citep{sachan2021emdr}. FAR differs in how this candidate score is parameterized.
Through \Cref{eq:multi_cue_score}, the posterior supervision acts on both
cue-specific relevance and query-dependent cue weights.
For a fixed training example and candidate pool,
\begin{align*}
\nabla_\phi
\KL\!\left(
\sg[q_{\theta,\phi}^{\mathcal A_t}]
\,\middle\|\,
r_\phi^{\mathcal A_t}
\right)
&=
\bbE_{\bmo_i\sim r_\phi^{\mathcal A_t}}
\!\left[
\nabla_\phi s_{\phi,i}
\right]
-
\bbE_{\bmo_i\sim q_{\theta,\phi}^{\mathcal A_t}}
\!\left[
\nabla_\phi s_{\phi,i}
\right]
\\
&=
\sum_m
\lambda_{\phi,t}^m
\left(
\bbE_{r_\phi^{\mathcal A_t}}
\!\left[
\nabla_\phi\widetilde{s}_{\phi,i}^m
\right]
-
\bbE_{q_{\theta,\phi}^{\mathcal A_t}}
\!\left[
\nabla_\phi\widetilde{s}_{\phi,i}^m
\right]
\right)
\\
&\quad+
\sum_m
\left(
\bbE_{r_\phi^{\mathcal A_t}}
\!\left[
\widetilde{s}_{\phi,i}^m
\right]
-
\bbE_{q_{\theta,\phi}^{\mathcal A_t}}
\!\left[
\widetilde{s}_{\phi,i}^m
\right]
\right)
\nabla_\phi\lambda_{\phi,t}^m ,
\end{align*}
where $q_{\theta,\phi}^{\mathcal A_t}$ is treated with stop-gradient.
The decomposition exposes two learning pathways: one through the cue-specific
relevance scores and another through the query-dependent cue weights.
Thus, future-aware predictive credit shapes both how memories are scored under
each cue and how strongly each cue influences recall for the current query.
Thus, a cue is reinforced when it favors memories that receive greater
future-aware predictive credit.
With a single cue, the cue-reliability term vanishes, recovering the
\emdr-style candidate update; with multiple cues, the same predictive signal
also learns which cues to trust for the current recall query.

This approximation affects only retrieval credit assignment.
The world model is still trained on the complete recalled context and can
therefore reason jointly over multiple memories.
Candidate-level training does not explicitly score higher-order selection effects,
such as memories that become informative only when recalled together, while such
interactions remain represented by the set-level model in
\Cref{eq:latent_context_model,eq:future_posterior}.
For our video instantiation, singleton likelihood is replaced by the
diffusion-based utility estimator below, and the candidate pools are constructed
from the temporally structured recall policy described next.

\section{Additional Implementation Details}

\subsection{Implementation Details of FAR for the Video Diffusion Instantiation}
\label{app:video_far_details}

This section provides the architectural and optimization details of the video
instantiation described in \Cref{sec:video_instantiation}.
We specify the diffusion objective, retrieval-score parameterization, adaptive
cue fusion, predictive-utility estimator, and training procedure.

\paragraph{Video diffusion world model and objective.}
We parameterize $p_\theta$ with a diffusion transformer (DiT)
~\citep{peebles2023dit}.
Retrieved context-frame tokens are processed jointly with noisy target-frame
tokens, using alternating intra-frame and global self-attention following
VGGT~\citep{wang2025vggt}.
The world model predicts the next observation from the current observation and
action together with the recalled context frames.
Time and pose metadata associated with recalled frames are treated as part of
their context representation.

Below, we make the diffusion objective explicit because it serves both to train the
world model and, through its standard connection to likelihood-based diffusion
training, to provide a tractable surrogate for predictive utility.
Let $\bmx_0$ denote the clean diffusion representation of the target observation
$\bmo_{t+1}$.
At diffusion time $\tau$, we sample
$\bm{\epsilon}\sim\mathcal N(\bm{0},\bm{I})$ and construct
\begin{align*}
\bmx_\tau
=
\alpha_\tau \bmx_0
+
\sigma_\tau \bm{\epsilon},
\end{align*}
where $\alpha_\tau$ and $\sigma_\tau$ define the noise schedule.
Given recalled context $\cC_t$ and prediction query $\cQ_t$, the DiT outputs
$\bmf_\theta(\bmx_\tau,\tau;\cC_t,\cQ_t)$.
Let $\bmy_\tau$ denote the corresponding training target under the chosen
diffusion parameterization, e.g., $\bmy_\tau=\bm{\epsilon}$ for noise prediction.
The per-perturbation loss is
\begin{align*}
\ell_{\mathrm{diff}}
\!\left(
\theta;
\bmo_{t+1},
\cC_t,
\cQ_t,
\tau,
\bm{\epsilon}
\right)
=
w(\tau)
\left\|
\bmf_\theta(\bmx_\tau,\tau;\cC_t,\cQ_t)
-
\bmy_\tau
\right\|_2^2,
\end{align*}
where $w(\tau)$ is the weighting used by the diffusion objective.

The world model is trained on the complete recalled context
$\cC_{t,\ctopk}$:
\begin{align}
\cL_{\wm}(\theta)
\coloneqq
\bbE
\left[
\ell_{\mathrm{diff}}
\!\left(
\theta;
\bmo_{t+1},
\cC_{t,\ctopk},
\cQ_t,
\tau,
\bm{\epsilon}
\right)
\right],
\label{eq:world_model_diffusion_loss}
\end{align}
where the expectation is over training examples and diffusion perturbations.

Retrieval cues affect which observations enter $\cC_{t,\ctopk}$.
In particular, visual and audio retrieval representations are used for memory
selection and are not passed to the world model as additional generation
conditions.
This allows the retriever to exploit additional cues without changing the
generator's conditioning interface.

\paragraph{Cue-specific retrieval scores and cacheable representations.}
For the metadata cue, consisting of time and pose, we use a lightweight scoring
network:
\begin{align*}
s_{\phi,i}^{\mathrm{meta}}
=
\tanh\!\left(
\mlp_\phi
\left(
[
\bmz_i^{\mathrm{meta}},
\bmz_t^{\mathrm{meta}},
\bma_t
]
\right)
\right).
\end{align*}

For a high-dimensional cue
$m\in\{\vision,\audio\}$,
we encode each historical cue as a memory key and the current cue as an
action-conditioned query:
\begin{align*}
\bmk_i^m
&\coloneqq
\enc_m(\bmz_i^m,\varnothing),
&
\bm{h}_t^m
&\coloneqq
\enc_m(\bmz_t^m,\bma_t),
\\
\bmq_t^m
&\coloneqq
\mlp_{\phi,m}(\bm{h}_t^m),
&
s_{\phi,i}^m
&\coloneqq
\cossim(\bmk_i^m,\bmq_t^m),
\end{align*}
where $\varnothing$ denotes a null action.
We use a ViT~\citep{dosovitskiy2021vit}, inject the action or null token through
AdaLN conditioning~\citep{peebles2023dit}, and use a learnable
\texttt{[CLS]} token as the retrieval representation.
Audio follows the same construction with an audio-specific encoder.

To make historical keys cacheable, we first pretrain the high-dimensional
encoders with a predictive contrastive objective.
The action-conditioned embedding $\bm{h}_t^m$ is trained to match the subsequent
key $\bmk_{t+1}^m=\enc_m(\bmz_{t+1}^m,\varnothing)$.
Given $N-1$ negative keys $\{\bmk_j^{m,-}\}_{j=1}^{N-1}$, we minimize
\begin{align*}
\cL_{\nce}^{m}
=
-\bbE
\left[
\log
\frac{
\exp\!\left(
\cossim(\bm{h}_t^m,\bmk_{t+1}^m)/\tau_{\mathrm{NCE}}
\right)
}{
\exp\!\left(
\cossim(\bm{h}_t^m,\bmk_{t+1}^m)/\tau_{\mathrm{NCE}}
\right)
+
\sum_{j=1}^{N-1}
\exp\!\left(
\cossim(\bm{h}_t^m,\bmk_j^{m,-})/\tau_{\mathrm{NCE}}
\right)
}
\right],
\end{align*}
where $\tau_{\mathrm{NCE}}$ is the contrastive temperature.
After pretraining, the encoder is frozen and historical keys are computed once
and cached.
The lightweight query adapter $\mlp_{\phi,m}$ remains trainable through FAR.

\paragraph{Adaptive cue weighting.}
Because different cues may produce scores on different scales, we standardize
each cue over the current episodic memory:
\begin{align*}
\widetilde{s}_{\phi,i}^{m}
=
\frac{
s_{\phi,i}^{m}-\mu_t^m
}{
\sigma_t^m+\delta
},
\end{align*}
where $\mu_t^m$ and $\sigma_t^m$ are the mean and standard deviation of
$\{s_{\phi,i}^m\}_{\bmo_i\in\cM_{t-1}}$, and $\delta>0$ is a small numerical
constant.

The query-dependent weights in \Cref{eq:multi_cue_score} are produced from
summary statistics of each cue:
\begin{align*}
\lambda_{\phi,t}^{m}
\propto
\exp\!\left(
\bmw_\phi^\top
\left[
\begin{array}{c}
\max_{\bmo_i\in\cM_{t-1}} s_{\phi,i}^m \\
\mu_t^m \\
\sigma_t^m \\
\bm e_m \\
\Delta t
\end{array}
\right]
\right),
\qquad
\sum_m \lambda_{\phi,t}^{m}=1,
\end{align*}
where $\bm e_m$ is a one-hot cue-type indicator and $\Delta t$ is the simulation
stride.
The normalization is taken over the cues available in the current environment.
Because these score statistics depend on the current query, FAR can vary the
relative importance of the available cues from one retrieval query to another.

\paragraph{Diffusion-based predictive utility.}
The exact singleton utility $u_{\theta,i}
=
\log
p_\theta
\!\left(
\bmo_{t+1}
\mid
\{\bmo_i\},
\cQ_t
\right)$ in \Cref{eq:singleton_utility} requires the
conditional log-likelihood, which is expensive to evaluate
for a diffusion model.
Diffusion models are instead trained with denoising objectives that arise from,
or are closely related to, variational likelihood training~\citep{ho2020denoising,kingma2021variational,song2021maximum,lai2025principles}.
We therefore use the negative diffusion prediction loss as a tractable surrogate
for comparing the predictive utility of candidate memories.
For candidate memory $\bmo_i$, we evaluate the realized future using the
singleton context $\{\bmo_i\}$:
\begin{align}
\widehat{u}_{\theta,i}
\coloneqq
-
\frac{1}{S}
\sum_{s=1}^{S}
\ell_{\mathrm{diff}}
\!\left(
\theta;
\bmo_{t+1},
\{\bmo_i\},
\cQ_t,
\tau_s,
\bm{\epsilon}_s
\right).
\label{eq:diffusion_proxy}
\end{align}
We use $S=4$.
A memory receives greater predictive credit when conditioning on it yields lower
diffusion prediction loss for the future that actually occurred.

All candidates in the same training example are evaluated using the same
perturbations
$\{(\tau_s,\bm{\epsilon}_s)\}_{s=1}^{S}$.
This common randomness reduces variation caused by diffusion sampling and makes
candidate comparisons more directly reflect the recalled memory.
We substitute $\widehat{u}_{\theta,i}$ for $u_{\theta,i}$ in
\Cref{eq:candidate_posterior}, with the resulting posterior treated with
stop-gradient during the retriever update.

Singleton contexts are used for retrieval credit assignment, while the world
model is trained on the complete recalled context
$\cC_{t,\ctopk}$ through \Cref{eq:world_model_diffusion_loss}.

\paragraph{Temporally structured recall and candidate pools.}
Nearby video frames are often redundant, so unconstrained Top-$K$ selection may
allocate several context slots to the same short temporal region.
We partition $\cM_{t-1}$ into non-overlapping chunks
$\{\cB_{t-1,j}\}_j$ of $n_{\chunk}$ consecutive observations and allow at most
one recalled observation from each chunk:
\begin{align}
\cC_{t,\ctopk}
\coloneqq
\underset{\cC_t\subseteq\cM_{t-1}}{\arg\max}
\;
\sum_{\bmo_i\in\cC_t}s_{\phi,i}
\quad
\text{s.t.}
\quad
|\cC_t|=K,
\qquad
|\cC_t\cap\cB_{t-1,j}|\leq1
\;\; \forall j.
\label{eq:chunk_topk}
\end{align}
Equivalently, we retain the highest-scoring observation from each chunk and then
select the $K$ highest-scoring representatives.

For retriever training, we use a global candidate pool containing the selected
representatives and a local pool containing one of their corresponding chunks:
\begin{align*}
\mathcal A_t^{\mathrm{global}}
=
\cC_{t,\ctopk},
\qquad
\mathcal A_t^{\mathrm{local}}
=
\cB_{t-1,j}.
\end{align*}
The global pool assigns predictive credit across temporal regions, while the
local pool distinguishes observations within one selected region.
For each pool, we compute \Cref{eq:diffusion_proxy} and apply the future-aware
candidate target in \Cref{eq:candidate_posterior} with the loss in
\Cref{eq:candidate_distillation}.

\begin{algorithm}[t]
\caption{Training Future-Aware Recall with a Video Diffusion World Model}
\label{alg:joint_training}
\begin{algorithmic}[1]
\Require retrieval size $K$, chunk size $n_{\chunk}$,
retriever interval $N_{\lazy}$, utility samples $S$
\For{training step $n=1,2,\ldots$}
    \State Sample
    $(\bmo_{t+1},\cM_{t-1},\cZ_{t-1},\cR_t,\cQ_t)$
    \State Compute and fuse cue-specific scores using
    \Cref{eq:multi_cue_score}
    \State Select $\cC_{t,\ctopk}$ using \Cref{eq:chunk_topk}
    \State Sample $(\tau,\bm{\epsilon})$ and update $\theta$ using
    $\ell_{\mathrm{diff}}
    (\theta;\bmo_{t+1},\cC_{t,\ctopk},\cQ_t,\tau,\bm{\epsilon})$
    \If{$n \bmod N_{\lazy}=0$}
        \State Set
        $\mathcal A_t^{\mathrm{global}}\gets\cC_{t,\ctopk}$
        \State Sample a represented chunk $\cB_{t-1,j}$ and set
        $\mathcal A_t^{\mathrm{local}}\gets\cB_{t-1,j}$
        \State Sample shared perturbations
        $\{(\tau_s,\bm{\epsilon}_s)\}_{s=1}^{S}$
        \For{$\mathcal A_t
        \in
        \{\mathcal A_t^{\mathrm{global}},
          \mathcal A_t^{\mathrm{local}}\}$}
            \State Evaluate $\widehat{u}_{\theta,i}$ for
            $\bmo_i\in\mathcal A_t$ using \Cref{eq:diffusion_proxy}
            \State Construct \Cref{eq:candidate_posterior} using
            $\widehat{u}_{\theta,i}$
            \State Compute the corresponding loss in
            \Cref{eq:candidate_distillation}
        \EndFor
        \State Update $\phi$ with the sum of the global and local losses
    \EndIf
\EndFor
\end{algorithmic}
\end{algorithm}

\paragraph{Joint training and inference.}
The world model is updated at every optimization step using
\Cref{eq:world_model_diffusion_loss}.
Candidate-wise utility evaluation requires additional DiT forward passes, so we
update the retriever once every $N_{\lazy}$ world-model steps.
At each retriever update, the global and local candidate pools are evaluated
using shared diffusion perturbations, and $\phi$ is updated with the sum of their
candidate-distillation losses.

At inference, FAR requires neither the future observation nor predictive-utility
evaluation.
Historical retrieval keys remain cached; FAR computes the available cue-specific
scores, adaptively fuses them, and recalls $\cC_{t,\ctopk}$ for world-model
prediction.

\subsection{Dataset Details}
\label{app:dataset_details}

\Cref{tab:dataset_stats} summarizes the corpora used in our experiments.
All datasets follow an exploration-or-outbound phase followed by a return or reveal phase, but differ in what information must be recovered from memory.

\begin{table*}[t]
\centering
\small
\setlength{\tabcolsep}{5pt}
\begin{tabular}{lccccc}
\toprule
& LoopNav & SoundSpaces v1 & SoundSpaces v2 & AI2-THOR & AI2-THOR-dyn \\
\midrule
Environment
& Minecraft & Matterport3D & Matterport3D & iTHOR & iTHOR \\
Episodes
& $19{,}200$ & $14{,}425$ & $14{,}425$ & $8{,}125$ & $1{,}202$ \\
Scenes
& $120$ & $81$ & $81$ & $119$ & $7$ \\
Train/test split
& episode-level & scene-disjoint & scene-disjoint & scene-disjoint & episode-level \\
Train/test scenes
& $120/120$ & $57/24$ & $57/24$ & $95/24$ & $7/7$ \\
Frames
& $12.62$M & $9.28$M & $10.95$M & $9.23$M & $1.85$M \\
FPS
& $20$ & $10$ & $10$ & $10$ & $10$ \\
Resolution
& $360\!\times\!640$ & $256\!\times\!256$ & $256\!\times\!256$ & $256\!\times\!256$ & $256\!\times\!256$ \\
Additional signal
& -- & binaural audio & binaural audio & interactions & second-agent state \\
\bottomrule
\end{tabular}
\caption{\textbf{Dataset statistics.}
SoundSpaces and AI2-THOR use scene-disjoint evaluation, whereas LoopNav and the reported AI2-THOR-dyn experiments use episode-level splits with the same scenes appearing in training and test.}
\label{tab:dataset_stats}
\end{table*}

\paragraph{LoopNav.}
We use the released LoopNav dataset~\citep{lian2025loopnav}, containing $19{,}200$ loop-navigation episodes across $120$ procedurally generated Minecraft village scenes.
Episodes follow ABA or ABCA routes over four range tiers, and the frame at which the released goal metadata turns toward home defines the boundary between the outbound and return legs.
We use the outbound leg as episodic memory and evaluate prediction over the return leg.
Our $80/20$ split is performed over episodes without a scene constraint, so all $120$ scenes appear in both training and test and the benchmark measures generalization to unseen trajectories within known worlds.

\paragraph{SoundSpaces.}
We construct two paired SoundSpaces corpora~\citep{chen2020soundspaces} from Matterport3D environments, each containing $14{,}425$ episodes over $81$ eligible scenes with a scene-disjoint split of $57$ training and $24$ test scenes.
An agent traverses an ABA or ABCA loop while two to four static sound sources play continuously, with spatialized binaural audio rendered at $16$ kHz throughout the trajectory.
The outbound portion serves as memory and the final return leg is a freshly planned shortest path back to the starting location.
SoundSpaces v1 performs $360^\circ$ scans only at loop endpoints, whereas v2 adds deterministic turnaround and mid-path scans triggered by region changes, elevation changes, or long intervals without a scan.
The two variants use the same episode manifest, seeds, waypoints, sound sources, clips, and gains, so matched v1/v2 comparisons isolate the effect of denser visual coverage.
No scans are inserted into the return leg.

\paragraph{AI2-THOR.}
We construct $8{,}125$ single-room iTHOR episodes~\citep{kolve2017ai2thor} over $119$ scenes, using $95$ scenes for training and $24$ held-out scenes for evaluation.
During exploration, the agent visits object stations, opens containers, and moves objects between surfaces and containers before revisiting the same stations during the return phase.
Station inspections use reproducible canonical viewpoints, creating same-pose observations that can depict different world states before and after an interaction.
Moved objects are hidden while carried, so the updated state must be inferred from the interaction history rather than visually tracked in transit.

\paragraph{AI2-THOR-dyn.}
We additionally construct $1{,}202$ corridor episodes over seven iTHOR rooms, including $1{,}002$ episodes with a moving second agent and $200$ frame-aligned static controls.
The observer watches a corridor through a constrained aperture while the second agent repeatedly crosses between its two ends, then traverses a door and looks toward both ends.
The observation post is verified to reveal the crossing agent at the aperture while hiding both corridor ends, and quality control rejects episodes containing visibility leaks outside the intended crossing windows.
The second agent alternates ends on each crossing, so its final location depends on the observed crossing history rather than the current view alone.
The reported experiments follow the seed-based split used by our training manifest, with $800$ acted episodes for training and $202$ acted episodes.
All seven rooms therefore appear in both splits, so these results measure generalization to unseen trajectories within known rooms rather than held-out-room generalization.

\subsection{Training Details}
\label{app:training_details}

\paragraph{Latent preprocessing.}
All corpora are trained in latent space with a frozen tokenizer, and frame latents are precomputed once before world-model training.
LoopNav uses the ViTVAE tokenizer from OASIS~\citep{oasis2024}, producing $16\times18\times32$ latents from $360\times640$ RGB frames with latent scale $0.07843137255$.
SoundSpaces and both AI2-THOR corpora use the SDXL VAE~\citep{podell2024sdxl} with the fp16 fix, producing $4\times32\times32$ latents from $256\times256$ RGB frames with latent scale $0.13025$.

\begin{table}[t]
\centering
\small
\setlength{\tabcolsep}{5pt}
\begin{tabular}{lc}
\toprule
Setting & Value \\
\midrule
DiT depth / hidden / heads & $12 / 768 / 12$ \\
Patch size & $2$ \\
Training context size & $4$ (current $+3$ recalled) \\
Diffusion steps & $1000$ \\
Noise schedule & Linear \\
Prediction target & $\boldsymbol{\epsilon}$-prediction \\
Variance & Learned \\
Optimizer & AdamW \\
Generator learning rate & $10^{-4}$ \\
Retriever learning rate & $10^{-4}$ \\
AdamW $\epsilon$ / weight decay & $10^{-6} / 0$ \\
Warm-up & $5$k steps, $0.2\times \rightarrow 1\times$ \\
Gradient clipping & $1.0$ \\
EMA decay & $0.9999$ \\
Utility noise levels & $4$ \\
Training precision & fp32 \\
\bottomrule
\end{tabular}
\caption{\textbf{Shared model and optimization settings.}
These settings are used across all corpora unless stated otherwise.}
\label{tab:shared_hparams}
\end{table}

\begin{table*}[t]
\centering
\small
\setlength{\tabcolsep}{4pt}
\resizebox{\textwidth}{!}{
\begin{tabular}{lccccc}
\toprule
& LoopNav & SoundSpaces v1 & SoundSpaces v2 & AI2-THOR & AI2-THOR-dyn \\
\midrule
Training pool size
& $200$ & $200$ & $200$ & $400$ & $400$ \\
Prediction horizon (frames)
& $128$ & $128$ & $128$ & $32$ & $32 / 135^{\dagger}$ \\
Goals per observation
& $4$ & $4$ & $4$ & $4$ & -- \\
Retrieval chunk size
& $10$ & $10$ & $10$ & $10$ & $30$ \\
FAR retrieval cues
& Meta+Vision & Meta+Audio & Meta+Audio & Meta+Vision & Meta / Meta+Agent \\
Metadata MLP layers
& $2$ & $2$ & $2$ & $3$ & $3$ \\
Retriever update interval
& $20$ & $20$ & $20$ & $20$ & $7$ \\
Event-anchored sampling
& -- & -- & -- & $0.30$ & $0.35$ \\
Change-mask weight
& -- & -- & -- & $0.3$ & $5.0$ \\
FAR effective batch size
& $128$ & $64$ & $64$ & $64$ & $64$ \\
\bottomrule
\end{tabular}}
\caption{\textbf{Dataset-specific training settings.}
The FAR cue row gives the signals used for retrieval; visual and audio retrieval representations are not passed to the world model as generation conditions.
$^{\dagger}$AI2-THOR-dyn uses a $135$-frame horizon for reveal-anchored pairs and $32$ frames otherwise.}
\label{tab:training_hparams}
\end{table*}

\paragraph{Optimization and sampling.}
The learning rate is linearly warmed from $0.2\times$ to its target value over the first $5{,}000$ steps and held constant thereafter.
We clip the global gradient norm to $1.0$, maintain an EMA with decay $0.9999$, and evaluate the EMA model every $2{,}500$ steps.
FAR uses an effective batch size of $128$ on LoopNav and $64$ on the remaining corpora.
The Temporal baseline uses effective batch sizes of $112$ on LoopNav and $63$ on SoundSpaces, while the remaining baseline arms match the corresponding FAR batch size.

\paragraph{Training pairs and memory pools.}
The dataset-specific settings are summarized in \Cref{tab:training_hparams}.
For LoopNav, training queries are sampled from the return leg while the retrievable pool is restricted to the outbound leg.
SoundSpaces uses the corresponding outbound portion as memory for prediction on the return trajectory.
AI2-THOR restricts the retrievable pool to the pre-evaluation exploration phase.
With probability $0.30$, an AI2-THOR training example is replaced by an event-anchored pair centered on an evaluation-phase open/close event, providing direct supervision near changed-state reveals.
AI2-THOR-dyn uses reveal-anchored sampling with probability $0.35$; anchored examples place the current frame during the door transit before the reveal and use the longer $135$-frame prediction range.

\paragraph{Retriever initialization and updates.}
The high-dimensional FAR retrieval encoders are initialized by contrastive pretraining before joint world-model training, and cached historical keys are exported from the same checkpoint used to initialize the corresponding query tower.
The QueryViT has six transformer layers, hidden dimension $384$, $12$ attention heads, and key dimension $256$.
Pretraining uses eight positive and $24$ negative keys per query.
The LoopNav, SoundSpaces, and AI2-THOR retrievers are pretrained for $400$k steps with learning rate $2\times10^{-4}$, while AI2-THOR-dyn uses a $12.5$k-step object-centric fine-tuning stage with learning rate $5\times10^{-5}$.
For SoundSpaces, positive pairs are selected by spatial proximity without a heading constraint so that a return-phase query can match an outbound observation seen from the opposite heading, and negatives are drawn from the same episode to avoid an episode-identity shortcut.
During joint training, the world model is updated at every optimization step from the context selected by the retriever.
Future-aware retriever updates use four stratified diffusion noise levels with shared perturbations across candidates and combine the global and within-chunk retrieval losses with equal weight.
LoopNav, both SoundSpaces variants, and AI2-THOR update the retriever every $20$ world-model steps, while AI2-THOR-dyn uses an interval of $7$ steps and applies the retriever update only to reveal-anchored examples.

\paragraph{Change-focused supervision in AI2-THOR.}
Because object interactions often affect only a small image region, AI2-THOR reweights the diffusion loss using a per-goal change mask.
The mask is normalized to preserve unit mean loss weight and uses coefficient $0.3$ with a maximum weighted-area fraction of $0.15$.
The same mask weighting is applied to every AI2-THOR retrieval arm so that comparisons remain matched.
AI2-THOR-dyn instead uses the second-agent silhouette as the loss mask with coefficient $5.0$ and the same maximum fraction, and applies this weighting to both the diffusion reader loss and FAR's predictive-utility reward.

\paragraph{Second-agent supervision in AI2-THOR-dyn.}
Each historical frame carries the second agent's visibility and egocentric floor position together with the same quantities $1.5$ seconds earlier, allowing a recalled slot to encode local motion.
An attention-pooling predictor over the recalled slots predicts second-agent presence and position using binary cross-entropy for presence and a masked position loss with weight $0.05$.
This auxiliary loss enters the training objective with weight $1.0$, and the same agent-state term is included in the candidate reward used for retriever training.
During training, the diffusion model receives the ground-truth goal agent state through teacher forcing, while inference uses the predictor output as described below.

\subsection{Inference Details}
\label{app:inference_details}

\begin{table*}[t]
\centering
\small
\setlength{\tabcolsep}{4pt}
\resizebox{\textwidth}{!}{
\begin{tabular}{lccccc}
\toprule
& LoopNav & SoundSpaces v1 & SoundSpaces v2 & AI2-THOR & AI2-THOR-dyn \\
\midrule
Evaluation checkpoint
& $700$k & $700$k & $700$k & $700$k & $100$k \\
Evaluation protocol
& rollout & rollout & rollout & rollout & both-ends probe \\
Stride (frames / seconds)
& $4 / 0.2$ & $16 / 1.6$ & $16 / 1.6$ & $4 / 0.4$ & -- \\
Context size
& $12$ & $12$ & $12$ & $12$ & $4$ \\
Retrieval pool size
& $1000$ & $1000$ & $1000$ & $1000$ & -- \\
Retrieval pool phase
& outbound & outbound & outbound & exploration & pre-reveal history \\
Retrieval chunk size
& $10$ & $10$ & $10$ & $10$ & $30$ \\
Sampler
& DDIM-$20$ & DDIM-$20$ & DDIM-$20$ & DDIM-$20$ & DDIM-$20$ \\
\bottomrule
\end{tabular}}
\caption{\textbf{Dataset-specific inference settings.}
The four rollout benchmarks expand the context from $4$ frames during training to $12$ frames at evaluation and query up to $1000$ real memory frames.}
\label{tab:inference_hparams}
\end{table*}

\paragraph{Autoregressive rollout protocol.}
For LoopNav, SoundSpaces, and AI2-THOR, we evaluate with replace-tail autoregressive rollouts along the ground-truth camera path using the settings in \Cref{tab:inference_hparams}.
At each step, the world model generates the next latent and the generated latent replaces only the current tail frame for the following step.
The retrievable episodic memory remains composed entirely of real observations from the outbound or exploration phase, so rollout errors do not contaminate the stored memory.
Camera poses are re-centered on the new current frame at every step using the same coordinate convention as in training.
The retrieval module is queried at every generation step.
Every method on a given clip uses the same diffusion-noise seed and target trajectory.

\paragraph{Generation conditioning.}
At each rollout step, the generator receives the current frame, recalled context frames, per-frame pose and trajectory-time conditioning, and the target action represented as a local pose displacement together with the generation stride.
For AI2-THOR, the generator additionally receives the interaction labels occurring between the current and target states together with the absolute door fraction at the target.
These interaction labels are supplied by the evaluation trajectory rather than generated by the model because camera motion alone does not specify object interactions.
Visual and audio retrieval representations are used only to rank episodic memory and are not passed to the generator as additional conditioning variables.
Thus, SoundSpaces uses binaural audio to select context while generation remains conditioned on the selected visual frames together with pose, time, and action information.

\paragraph{Evaluation operating point.}
All reported autoregressive rollouts use EMA weights and DDIM~\citep{song2021denoising} with $20$ sampling steps.
Relative to training, the rollout protocol increases the context size from $4$ to $12$ and the retrieval pool from $200$ or $400$ frames to as many as $1000$ frames, uniformly across compared methods within each corpus.
WorldMem retains a $200$-frame pool cap on LoopNav and SoundSpaces to match its training configuration, while AI2-THOR uses the full evaluation pool.
The reported batteries contain $3{,}796$ LoopNav clips, $1{,}167$ clips for each SoundSpaces variant, and $1{,}309$ AI2-THOR clips.

\paragraph{AI2-THOR-dyn inference.}
AI2-THOR-dyn is evaluated with a one-step both-ends prediction probe rather than an autoregressive rollout.
For each memory-dependent test trajectory, the two possible reveal ends are generated separately from the same current observation using the same diffusion-noise seed, so the generations differ only in the goal and retrieved memory.
The world model uses the agent-state predictor's own binarized output at inference without teacher forcing.
Real memory frames retain their ground-truth second-agent detections, while generated states are threaded forward using the predictor output together with the corresponding state from $1.5$ seconds earlier.
A three-layer latent-space detector determines whether the generated frame contains the second agent, with at least $10$ firing latent cells counted as a positive rendering.
On real test frames, this detector has $0.906$ recall and $0.070$ false-positive rate at the chosen threshold.

\end{document}

%% file: math_commands.tex
\usepackage{amsmath,amsfonts,bm}

\def\eqref#1{equation~\ref{#1}}

\def\1{\bm{1}}

\DeclareMathAlphabet{\mathsfit}{\encodingdefault}{\sfdefault}{m}{sl}
\SetMathAlphabet{\mathsfit}{bold}{\encodingdefault}{\sfdefault}{bx}{n}

\newcommand{\KL}{\mathcal{D}_{\mathrm{KL}}}

